\documentclass[preprint,12pt]{elsarticle}

\usepackage[utf8]{inputenc}
\usepackage{amsmath}
\usepackage{bm}
\usepackage{url}
\usepackage{graphicx}
\usepackage{indentfirst}
\usepackage{mathtools}
\usepackage{amssymb}
\usepackage{amsthm}
\usepackage{natbib}
\usepackage{hyperref}
\usepackage{stackrel}
\usepackage{amsmath}
\usepackage{subfigure}
\usepackage{caption}

\DeclareMathOperator*{\argmin}{arg\,min}
\usepackage[algoruled,boxed,lined]{algorithm2e}

\newtheorem{lemma}{Lemma}
\newtheorem{corollary}{Corollary}
\newtheorem{theorem}{Theorem}

\journal{Neurocomputing}

\begin{document}

\begin{frontmatter}

 \title{Adaptive Hedging under Delayed Feedback}%\tnoteref{grant}}
%\tnotetext[grant]{The research was partially supported by the Russian Foundation for Basic Research grant 16-29-09649 ofi m.}
\author{Alexander Korotin}
\ead{a.korotin@skoltech.ru}
\author{Vladimir V'yugin}
\ead{v.vuygin@skoltech.ru}
\author{Evgeny Burnaev}
\ead{e.burnaev@skoltech.ru}
\address{Skolkovo Institute of Science and Technology}

\begin{abstract}
The  article  is  devoted to investigating the application of hedging strategies to online expert weight allocation under delayed feedback. As the main result we develop the General Hedging algorithm $\mathcal{G}$ based on the exponential reweighing of experts' losses. We build the artificial probabilistic framework and use it to prove the adversarial loss bounds for the algorithm $\mathcal{G}$ in the delayed feedback setting. The designed algorithm $\mathcal{G}$ can be applied to both countable and continuous sets of experts. We also show how algorithm $\mathcal{G}$ extends classical Hedge (Multiplicative Weights) and adaptive Fixed Share algorithms to the delayed feedback and derive their regret bounds for the delayed setting by using our main result.
\end{abstract}


\end{frontmatter}

\section{Introduction}

We consider the \textbf{Decision-Theoretic Online Learning} (DTOL) framework \cite{LiW94, FrS97,devroye2013prediction,cesa2007improved,kotlowski2018minimaxity,devroye2013prediction} which is closely related to the paradigm of
\textbf{prediction with expert advice} \cite{cesa-bianchi,Vov90,VoV98,Vovk1999,LiW94,adamskiy2012putting,BoW2002,KorBurAggrBase2018}. A \textbf{master algorithm} at every step ${t=1,\dots,T}$ of the game has to choose the weight allocation for a given pool of expert strategies (experts). We call this problem the \textbf{experts problem}. We investigate the \textbf{adversarial case}, i.e., no assumptions are made about the nature of the data (stochastic, deterministic, etc.).

The performance of the master algorithm is measured by the \textbf{regret} over the entire game. The regret $R_T$ is the difference between the cumulative loss of the online algorithm and the loss of some given comparator. A typical comparator is the best fixed expert in the pool or the best fixed convex linear combination of experts. The goal of the algorithm is to minimize the regret, i.e., $R_{T}\rightarrow \min$. 

In the classical online learning, the algorithm suffers loss of its decision at each step $t$ at the end of the same step the decision is made. In contrast to the classical scenario, we consider the \textbf{delayed feedback} learning. At each step $t$ of the game the algorithm makes a decision, and its result will be revealed only at the end of a time point $t+D_{t}$ (where $D_{t}\geq 0$ is some \textbf{delay}).

It turns out that there exists a wide range of algorithms for the non-delayed scenario ($D_{t}\equiv 0$). Almost all of them exploit the \textbf{follow-the-best-expert idea}: the better expert performed in the past, the higher relative weight is assigned to the expert. The pure \textbf{Follow the Leader} (FTL) strategy is well-known to have good performance in the stochastic setting\footnote{Pure Follow the Leader strategy is known to be the minimax in the simplest stochastic setting (experts' losses are i.i.d. between experts and time steps).}  \cite{kotlowski2018minimaxity}, but it may be inefficient when the data is generated by an adversary (see discussion in \cite{shalev2012online,de2014follow}).

\textbf{Follow the Perturbed Leader} (FTPL) algorithm \cite{kalai2005efficient} adds random noise to expert evaluation process. This prevents overfitting in adversarial setting. For example, exponential \cite{kalai2005efficient}, random-walk \cite{devroye2013prediction}, and dropout \cite{van2014follow} noise has been shown to achieve low expected regret for the experts problem.

\textbf{Follow the Regularized Leader}\footnote{Equivalently, Online Mirror Descent, see \cite{mcmahan2011follow}.} (FTRL) is a powerful algorithm from online convex optimization framework \cite{HazanOCO16}, \cite{shalev2012online}. The usage of the linear loss function on a simplex allows to deal with the experts problem. The quadratic regularization leads to Online Gradient Descent (OGD) algorithm \cite{shalev2012online}, the Entropic regularization provides Exponential Weights algorithm, also known as \textbf{Hedge} \cite{FrS97}.

The idea of \textbf{multiplicative weight updates} (MW) of Hedge algorithm is used in many successive algorithms (MW2 \cite{cesa2007improved}, Variation-MW \cite{hazan2010extracting}, Optimistic-MW \cite{chiang2012online}, AEG-Path and AMEG-Path \cite{steinhardt2014adaptivity} and other algorithms \cite{gaillard2014second,koolen2015second}). The main goal of such algorithms is to obtain the first or the second order regret bound (e.g. in terms of best expert's loss) or achieve improvement for easy-data. Also, some Hedge-based algorithms (AdaHedge \cite{erven2011adaptive}, Flip-Flop \cite{de2014follow}) are designed to be parameter-free.

Almost all  described algorithms provide $O(\sqrt{T})$ adversarial regret guarantees w.r.t. the best expert in the pool. Note that this bound is minimax optimal up to some multiplicative factor because $\Omega(\sqrt{T})$ is known to be the lower bound \cite{cesa-bianchi}.\footnote{More precisely, the lower bound is $\Omega(\sqrt{T\ln N})$, where $N$ is the number of experts in the finite pool.}

An important variant of the experts problem is to develop an \textbf{adaptive master algorithm}. Such an algorithm has to track the shifts (switches) of the best expert and achieve low \textbf{tracking regret} with respect to shifting sequences of experts.\footnote{Sometimes in online learning the term \textbf{adaptive} means that the algorithm dynamically changes its learning rate during the game. Please do not get confused.} There are many meta-approaches such as restarts \cite{AdaAda16,hazan2009efficient} or specialist experts \cite{freund1997using} to create adaptive algorithms from non-adaptive ones. However, the most recognizable approach is to use the \textbf{Fixed Share} extension for Hedge \cite{HeW98,cesa2012mirror,AdaAda16,BoW2002}.

When it comes to the delayed feedback setting, many of the above described \textbf{non-delayed algorithms} do not have theoretical guarantees of performance or do not even have a modification for the delayed feedback setting.

There exists a bunch of meta-algorithms that allow to produce a version for delayed feedback setting from the basic non-delayed version \cite{WeO2002,Mes2009,Mes2007,joulani2013online}. The roots of meta-approach lie in the work \cite{WeO2002}. The authors studied the setting under fixed known feedback delay $D$. They proved that the optimal (non-adaptive) algorithm is to run $D+1$ independent versions of the optimal non-delayed algorithm on $D+1$ disjoint time grids ${GR_{d}=\{t\mbox{ }|\mbox{ }t\equiv d\mbox{ }(\mbox{mod } D+1)\}}$ for $1{\le d\le D+1}$. Thus, the optimal worst-case adversarial regret is $(D+1)\cdot \Omega(\sqrt{\frac{T}{D+1}})=\Omega(\sqrt{T(1+D)})$. The described meta-approach was enhanced for the unknown and dynamic feedback delay in \cite{joulani2013online}. Their meta-algorithm \textbf{BOLD} (Black-box Online Learning with Delays) also runs independent copies of  the basic algorithm on disjoint time lines.

We call algorithms obtained by meta-approaches (such as BOLD) \textbf{replicated} algorithms. Whereas replicating is simple and in some cases is theoretically optimal, it has several obvious practical drawbacks. Firstly, it uses only part of the observed data at every step of the game. Secondly, separate replicating learning processes generated by the meta-algorithm do not even interact.

Non-adaptive algorithms based on FTRL and FTPL have several \textbf{non-replicated} adaptations for delayed feedback setting. The most straightforward ones are Delayed OGD \cite{LSZ2009}, Delayed FTPL and FTRL \cite{QuanDGD15} and FTRL with Memory \cite{AHS2015}. For the fixed and known feedback delay $D$ their best regret bound is $O(\sqrt{T(1+D)})$, which is optimal.

\textbf{In this work,} we aim to create an adaptive non-replicated algorithm for the delayed feedback setting. We base our research on the Hedge algorithm (and its adaptive extension Fixed Share), which is the state-of-the-art basis for many existing algorithms. In order to achieve the desired goal, we develop the general probabilistic framework for Hedge-based algorithms. Using this framework, we propose the \textbf{General Hedging Algorithm} $\mathcal{G}$, prove its loss bounds  both for delayed and non-delayed cases. As a corollary of the main result, we show how classical non-delayed Hedge and Fixed Share algorithms (as the cases of $\mathcal{G}$) can be extended to the delayed feedback setting and what regret bounds they have.

\vspace{1mm}

\noindent\textbf{The main contributions of this paper are:}
\begin{enumerate}
\item Developing the General Hedging algorithm $\mathcal{G}$ for the delayed feedback scenario which is applicable to both non-delayed and delayed online settings. Proving the algorithm's loss bound (and regret bound, for the case of a countable set of experts) in a general form.
\item Developing (for a finite number of experts) non-replicated versions of basic Hedge \cite{FrS97} and adaptive Fixed Share \cite{HeW98} algorithms (as special cases of algorithm $\mathcal{G}$) for the delayed feedback scenario as well as deriving their regret bounds.
\end{enumerate}

The General Hedging algorithm  $\mathcal{G}$ which we develop is motivated by the paper \cite{adamskiy2012putting}. In that work the authors considered  the special case of the prediction with experts' advice with the \textbf{logarithmic loss function}. For the traditional non-delayed scenario ($D_{t}\equiv 0$) they developed the \textbf{Bayesian Merging Algorithm} for mixing (averaging) experts' predictions. Their algorithm is based on the natural graphical model (similar to the one in Figure \ref{figure:model-general-n} of Section \ref{sec-algorithm}) implied by the probabilistic origin of the logarithmic loss function. 

In contrast to \cite{adamskiy2012putting}, we consider the decision-theoretic online learning scenario (hedging), which is more general than prediction with experts' advice.\footnote{Hedging scenario assumes that the learner has access only to losses of experts while in prediction with experts' advice the learner knows experts' predictions and observes true outcomes (the losses are computed by using the known loss function). Prediction with experts' advice can be reduced to Hedging by forgetting about the expert's predictions and using only the computed losses of the experts.} At the same time we investigate both non-delayed and delayed feedback settings. We build the artificial probabilistic framework for arbitrary bounded losses by using the \textbf{entropithication} transform (loss exponentiation, see e.g. \cite{grunwald2004game,van2015fast}), state the General Hedging algorithm $\mathcal{G}$ and prove its loss bound.

\vspace{1mm}

\noindent\textbf{The article is structured as follows:}

In Section \ref{sec-prelim} we give preliminary notions, describe the notation and the setting of the game of the delayed feedback experts' weights allocation.

In Section \ref{sec-algorithm} we describe the developed probabilistic framework, the main algorithm $\mathcal{G}$, and formulate the main Theorem \ref{theorem-main-loss-bound-D} about its loss bound. In Section \ref{sec-proofs} we prove the main theorem. 

In Section \ref{sec-examples} we provide the examples of the application of algorithm $\mathcal{G}$: Delayed Hedge in Subsection \ref{hedge-delayed}, Delayed Fixed Share in Subsection \ref{sec-fs}. 

% In Section \ref{sec-toy-experiment} we provide the results of the toy computational experiment. 

In Section \ref{sec-experiments} we conduct massive computational experiments and provide the detailed discussion of the results.

In \ref{sec-appendix-math} we provide the necessary mathematical background.

\section{Preliminaries}
\label{sec-prelim}

We use \textbf{bold} font to denote vectors (e.g. $\bm{w}\in \mathbb{R}^{M}$ for some integer $M$). In most cases, superscript is used for indexing elements of a vector (e.g. ${(w^{1},\dots,w^{N})=\bm{w}}$). Subscript is always used to indicate time (e.g. $l_{t}, R_{T}, w_{t}^{n}$).

We consider the online game of delayed hedging of a (finite or infinite) pool of experts. We use $\mathcal{N}$ to denote the pool and $n\in\mathcal{N}$ as an index of an expert. In this paper $\mathcal{N}$ is either a discrete set (e.g. ${\mathcal{N}=\{1,2,\dots, N\}}$) or a continuous subset of Euclidean space (e.g. ${\mathcal{N}=\mathbb{R}^{M}}$). By $\Delta(\mathcal{N})$ for a discrete (continuous) set $\mathcal{N}$ we denote all discrete (continuous) probability distributions on $\mathcal{N}$.

For convenience, we do all calculations in the paper assuming that $\mathcal{N}$ is a discrete countable set. All the results also hold true for the continuous $\mathcal{N}$ but sums over $n$ (e.g. ${\sum_{n\in\mathcal{N}}[\ldots]}$) should be replaced with corresponding intergrals (e.g. ${\int_{n\in\mathcal{N}}[\ldots]\cdot dn}$).

At each integer time step ${t=1,2,\dots, T}$ of the game the master (hedging) algorithm has to assign the weights $w^{n}_{t}$ to all experts ${n\in\mathcal{N}}$ so that 
$$\bm{w}_{t}=\{w_{t}^{n}\text{ for }n \in\mathcal{N}\}\in \Delta(\mathcal{N}).$$
At the end of the step $t+D_{t}$ (for integer $D_{t}\geq 0)$ experts reveal their losses ${\bm{l}_{t}=\{l_{t}^{n}\text{ for }n\in\mathcal{N}\}}$ at the step $t$. The loss of the algorithm's decision of the step $t$ is 
$$h_{t}=\sum_{n\in\mathcal{N}}l_{t}^{n}\cdot w_{t}^{n}=\langle \bm{w}_{t}, \bm{l}_{t}\rangle,$$ i.e.,
the average experts' loss w.r.t. $\bm{w}_{t}$. 

The sequence $D_{1},D_{2},\dots,D_{T}$ is called the sequence of  delays. For simplicity, we assume that ${t+D_{t}\leq T}$ for all $t=1,2,\dots,T$. In particular, $D_{T}=0$. We denote the set of all time indices of the losses revealed before the end of the step $t$ by ${\mathcal{D}_{t}=\{\tau|\tau+D_{\tau}\leq t\}}$. Also, we denote ${d\mathcal{D}_{t}=\mathcal{D}_{t}\setminus \mathcal{D}_{t-1}}$.

There are many scenarios on how the sequence $D_{t}$ is chosen (randomly, adversarially) and whether it is known to the learner in advance or not (see e.g. \cite{WeO2002,Mes2009,Mes2007,agarwal2011distributed,joulani2013online}). Yet, we do not specify the particular scenario, and consider the game in the general form.

In this work we assume that all the losses are bounded: ${l_{t}^{n}\in [0,H]}$ for all ${t=1,\dots,T}$ and ${n\in\mathcal{N}}$. This is a common assumption in online learning (see \cite{shalev2012online,HazanOCO16} or any other survey on online learning). The game setting is described by the following Protocol \ref{protocol}.

\begin{algorithm}
\SetAlgorithmName{Protocol}{empty}{Empty}
\SetKwInOut{Parameters}{Parameters}
\Parameters{Pool of experts $\mathcal{N}$; Game length $T$.}
\For{$t=1,2,\dots,T$}{
	Algorithm sets weights $\bm{w}_{t}\in\Delta(\mathcal{N})$\;
    \For{$\tau\in d\mathcal{D}_{t}$}{
      Nature reveals experts' losses $\bm{l}_{\tau}\in [0, H]^{\mathcal{N}}$\;
      Algorithm suffers loss $h_{\tau}=\langle \bm{w}_{\tau}, \bm{l}_{\tau}\rangle\in [0, H]$.
    }
 }
\caption{Online experts' weights allocation under delayed feedback}
\label{protocol}
\end{algorithm}

We use $H_{T}=\sum_{t=1}^{T}h_{t}$ and $L_{T}^{n}=\sum_{t=1}^{T}l_{t}^{n}$ to denote the cumulative (total) loss of the algorithm and expert $n\in\mathcal{N}$. 

The performance of the algorithm is measured by the (cumulative) regret. The regret is the difference between the cumulative loss of the algorithm and the cumulative loss of some given comparator. A typical approach is to compete with the best expert in the pool. The cumulative regret with respect to the best expert is
\begin{equation}
R_{T}=H_{T}-\min_{n\in\mathcal{N}}L_{T}^{n}.
\label{base-regret}
\end{equation}

The goal of the algorithm is to minimize the regret, i.e., ${R_{T}\rightarrow \min}$. In order to theoretically guarantee algorithm's performance, some upper bound is usually proved for the cumulative regret ${R_{T}\leq f(T)}$. 

In the basic setting \eqref{base-regret}, sub-linear upper bound $f(T)$ for the regret leads to the asymptotic performance of the algorithm equal to the performance of the best expert. More precisely, we have $\lim_{T\rightarrow \infty}\frac{R_{T}}{T}=0$.

\section{Generalized Hedging Algorithm}
\label{sec-algorithm}

In this section we describe the generalization $\mathcal{G}$ of the classical hedging algorithm based on exponential reweighing of experts' losses. The basic algorithm was introduced by \cite{FrS97}.

We investigate the adversarial case, i.e., no assumptions (stochastic, functional, etc.) are made about the nature of data (experts' losses). However, it turns out that in this case it is convenient to develop algorithms using some probabilistic framework.

\subsection{Probabilistic Framework}

Recall that $\bm{l}_{t}=\{l_{t}^{n}\text{ for }n\in\mathcal{N}\}$ is a dictionary of experts' losses at the step $t$. The framework that  we build implies that data is generated by some probabilistic model with hidden states. The graphical model is shown in Figure \ref{figure:model-general-n}.

\begin{figure}[!htb]
\begin{center}
\includegraphics[scale=0.6]{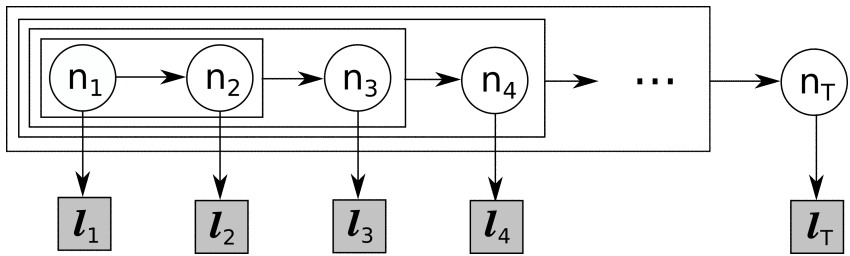}
\end{center}
\caption{General probabilistic model for data generation process}
\label{figure:model-general-n}
\end{figure} 

We suppose that there is some hidden sequence of experts $ n_{t}\in \mathcal{N}$ (for ${t=1,2,\dots,T}$) that generates the experts' losses $\bm{l}_{t}$. In particular, hidden expert $n_{t}$ at step $t$ is called \textbf{active expert}. The \textbf{conditional probability} to observe the vector $\bm{l}_{t}$ of experts' losses at the step $t$ is

\begin{equation}p(\bm{l}_{t}| n_{t})=p(l_{t}^{n_{t}}| n_{t})=\frac{e^{-\eta l_{t}^{n_{t}}}}{Z},
\label{observation-probability}
\end{equation}
where $\eta>0$ is some fixed learning rate and $Z=\int_{l\in [0, H]}e^{-\eta l}dl$ is the normalizing constant. Constant $Z$ is independent of both $n_{t}$ and $t$. The idea of conditional probability \eqref{observation-probability} is to assume that if at the step $t$ expert $n_{t}\in\mathcal{N}$ is active, then the loss vector $\bm{l}_{t}=(l_t^1, l_t^2,...,l_t^N)$ is not completely random, i.e. loss $l_t^{n_{t}}$ is random, while all the other components are deterministic (e.g. given by nature).\footnote{Another definition of conditional probability is also possible. All the elements $(l_t^1, l_t^2,...,l_t^N)$ can be considered as independent random variables. If expert $n_{t}\in\mathcal{N}$ is active, the probability of observing $l_t^{n_{t}}$ is equal to the current right-hand side of equation \eqref{observation-probability}. All the other losses are i.i.d. uniform variables on $[0, H]$. For the case of finite $\mathcal{N}$ the formula \eqref{observation-probability} is replaced by
\begin{equation}p(\bm{l}_{t}| n_{t})=p(l_{t}^{n_{t}}| n_{t})\times\bigg[\prod_{n\neq n_{t}}p(l_{t}^{n}|n_{t})\bigg]=\frac{e^{-\eta l_{t}^{n_{t}}}}{Z}\times \frac{1}{H^{N-1}},
\label{observation-probability-hard}
\end{equation}
i.e. has an additional denominating factor of $H^{N-1}$. However, for the infinite number of experts this approach requires a more detailed specification of probabilities in terms of measures, because the denominator becomes infinite. In Section \ref{sec-proofs} we will see that the exact value of the normalization constant $Z$ is important neither for the algorithm, nor for its regret bound. Thus, for convenience it is reasonable to consider the model \eqref{observation-probability}.}

For the first active expert $n_{1}$ some known \textbf{prior distribution} is given $p(n_{1})=p_{0}(n_{1})$. The sequence $(n_{1},\dots, n_{T})$ of active experts is generated step by step. For $t\in\{1,\dots,T-1\}$ each $ n_{t+1}$ is sampled from some known distribution $p(n_{t+1}|N_{t})$, where $N_{t}=(n_{1},\dots, n_{t})$.\footnote{In case $p( n_{t+1}|N_{t})=p(n_{t+1}| n_{t})$, we obtain a traditional Hidden Markov Process: the hidden state at step $t+1$ depends only on the previous hidden state at step $t$.} Thus, active expert $n_{t+1}$ depends on the previous experts $N_{t}$.

For every sequence of experts $N_{t}=(n_{1},n_{2},\dots,n_{t})$ we denote the cumulative loss of the sequence by $$L_{t}^{N_{t}}=\sum_{\tau=1}^{t}l_{\tau}^{n_{\tau}}.$$
For all $t$ we define the following lists of loss vectors:
\begin{equation}
    \bm{L}_{t}=(\bm{l}_{1},\dots,\bm{l}_{t}),\qquad {\bm{L}_{\mathcal{D}_{t}}=\{\bm{l}_{\tau}\text{ for }\tau\in \mathcal{D}_{t}\}}, \qquad {\bm{L}_{d\mathcal{D}_{t}}=\{\bm{l}_{\tau}\text{ for }\tau\in d\mathcal{D}_{t}\}}.
    \nonumber
\end{equation}

The considered probabilistic model is:
\begin{eqnarray}
p(N_{T},\bm{L}_{T})=p(N_{T})\cdot p(\bm{L}_{T}|N_{T})=
\bigg[p_{0}(n_{1})\prod_{t=2}^{T}p(n_{t}|N_{t-1})\bigg]\cdot \bigg[\prod_{t=1}^{T}p(\bm{l}_{t}| n_{t})\bigg].
\label{model-general}
\end{eqnarray}

The probability $p(N_{T})$ is that of hidden states (active experts).\footnote{The form 
$p({N}_{T})=p_{0}(n_{1})\prod_{t=2}^{t}p(n_{t}|{N}_{t-1})$ is used only for convenience and association with online scenario. It does not impose any restrictions on the type of probability distribution. In fact, $p(N_{T})$ may be any distribution on $\mathcal{N}^{T}$ of any form.}

\subsection{General Hedging Algorithm}

The hedging algorithm \ref{algorithm-main} is shown below. We denote it by $\mathcal{G}=\mathcal{G}(p)$ ($\mathcal{G}$ stands for \textbf{G}eneral), where $p$ indicates the probability distribution $p({N}_{T})$ of active experts to which the algorithm is applied.

\begin{algorithm}
\SetKwInOut{Parameters}{Parameters}
\Parameters{Pool of experts $\mathcal{N}$; Game length $T$; \\
Distribution on experts' sequences $p(\cdot)$.}
\For{$t=1,2,\dots,T$}{
	$\bm{w}_{t}\leftarrow p(n_{t}|\bm{L}_{\mathcal{D}_{t-1}})$\;
    \For{$\tau\in d\mathcal{D}_{t}$}{
      Nature reveals experts' losses $\bm{l}_{\tau}\in [0, H]^{\mathcal{N}}$\;
      Algorithm suffers loss $h_{\tau}\leftarrow\langle \bm{w}_{\tau}, \bm{l}_{\tau}\rangle$.
    }
 }
\caption{General Hedge Algorithm ($\mathcal{G}$)}
\label{algorithm-main}
\end{algorithm}

The idea of the algorithm $\mathcal{G}$ is simple: set the weight allocation $\bm{w}_{t}$ for the current step $t$ according to the posterior probability $p(n_{t}|\bm{L}_{\mathcal{D}_{t-1}})$ of the expert $n_{t}$ computed from the underlying probabilistic model. We illustrate this idea in Figure \ref{figure:model-posterior}.

\begin{figure}[!htb]
\begin{center}
\includegraphics[scale=0.6]{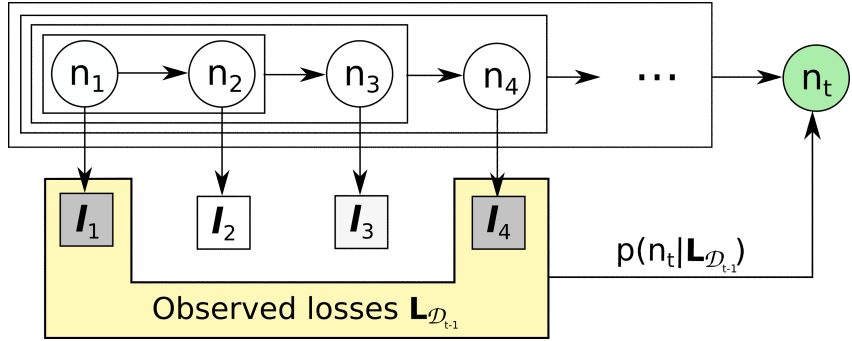}
\end{center}
\caption{The idea of General Hedge Algorithm. The weights $\bm{w}_{t}$ used at the step $t$ correspond to the posterior probability $p(n_{t}|L_{\mathcal{D}_{t-1}})$ of the hidden state at the current step}
\label{figure:model-posterior}
\end{figure} 

Consider a finite pool $\mathcal{N}=\{1,\dots,N\}$ and set $p_{0}(n_{1})\equiv \frac{1}{N}$ for all $n_{1}\in\mathcal{N}$. Consider the non-delayed scenario ($D_{t}\equiv 0$ for all $t$). If we use ${p(n_{1}=n_{2}=\dots=n_{T})\equiv 1}$, the experts' weights become $w_{t}^{n}\propto e^{-\eta L_{t-1}^{n}}$. The resulting algorithm $\mathcal{G}(p)$ turns to be classical non-delayed Hedge (for more detailed discussion see Subsection \ref{hedge-delayed}). Also, non-delayed Fixed Share is the case of $\mathcal{G}$ for specially chosen Markovian $p(\cdot)$ (see Subsection \ref{sec-fs}).

The \textbf{time and memory complexities} of the algorithm depend on the properties of the underlying distribution $p$. For Markovian models (when hidden state $n_{t}$ depends only on the previous state $n_{t-1}$ for all $t$) it is possible to provide linear in $T+\sum_{t=1}^{T}D_{t}$ schemes to compute weights (see Subsection \ref{sec-fs}) which require $O(N\cdot \max_t D_{t})$ memory. For the arbitrary $p(\cdot)$ time and memory complexity may be even exponential.

\subsection{Guarantees of Performance}

The algorithm has \textbf{theoretical guarantees of performance}. We state the following main theorem.

\begin{theorem}[Adversarial loss bound for algorithm $\mathcal{G}$]
\label{theorem-main-loss-bound-D}
Let $\mathcal{N}$ be a countable (or continuous) set of experts. Let $p(\cdot)$ be a discrete (or continuous) distribution on $\mathcal{N}^{T}$. Then for the hedging algorithm $\mathcal{G}$ applied to model $p$ with learning rate $\eta>0$ the following upper bound for the total loss over the entire game holds true:
\begin{eqnarray}
H_{T}\leq -\frac{1}{\eta}\ln \bigg[\mathbb{E}_{p({N}_{T})}\big[e^{-\eta L_{T}^{{N}_{T}}}\big]\bigg]+
\eta\frac{H^{2}}{8}T+\eta \big[\frac{H^{2}\cdot \sum_{t=1}^{T}D_{t}}{4}\big].
\label{main-loss-bound-convex}\end{eqnarray}
\end{theorem}

The proof of this theorem is given in Section \ref{sec-proofs}. Note that while the algorithm may seem to be designed for the stochastic setting, we apply it to the pure adversarial case\footnote{The only assumption is that the losses are bounded, i.e. $l_{t}^{n}\in[0,H]$ for all $t=1,\dots,T$ and $n\in\mathcal{N}$.} and obtain the loss guarantees. At the same time, the adversarial loss bound \eqref{main-loss-bound-convex} depends on the probability distribution $p(\cdot)$ for which the algorithm is designed.

One may wonder how Theorem \ref{theorem-main-loss-bound-D} is applied to estimate the regret, for example, the regret with respect to the best expert \eqref{base-regret}. If the set $\mathcal{N}$ of experts is countable, then the following simple corollary holds true.

\begin{corollary}[Adversarial regret bound for algorithm $\mathcal{G}$]
\label{corollary-regret-sequence}
If the set of experts $\mathcal{N}$ is countable, then under the conditions of Theorem \ref{theorem-main-loss-bound-D}, the regret with respect to any sequence ${{N}_{T}^{*}=(n_{1}^{*},n_{2}^{*},\dots, n_{T}^{*})\in\mathcal{N}^{T}}$ is

\begin{eqnarray}R_{T}({N}_{T}^{*})=H_{T}-L_{T}^{{N}_{T}^{*}}\leq 
-\frac{1}{\eta}\ln p({N}_{T}^{*})+\eta\frac{H^{2}}{8}T+\eta \big[\frac{H^{2}\cdot \sum_{t=1}^{T}D_{t}}{4}\big].
\label{cor-regret-bound}
\end{eqnarray}
\end{corollary}

\begin{proof} The corollary results from the following inequality for the expectation in the right-hand side of \eqref{main-loss-bound-convex}:
\begin{eqnarray}-\frac{1}{\eta}\ln \bigg[\mathbb{E}_{p({N}_{T})}\big[e^{-\eta L_{T}^{{N}_{T}}}\big]\bigg]\leq -\frac{1}{\eta}\ln \bigg[p({N}^{*}_{T})e^{-\eta L_{T}^{{N}_{T}^{*}}}\bigg]=
L_{T}^{{N}_{T}^{*}}-\frac{1}{\eta}\ln p({N}_{T}^{*}),
\nonumber
\end{eqnarray}
which leads to the desired bound.\end{proof}

If $\mathcal{N}$ is continuous under the conditions of Theorem \ref{theorem-main-loss-bound-D}, then the first term in the upper bound \eqref{main-loss-bound-convex} is represented by the integral (instead of a countable sum). It is not possible to extract a single summand as in the finite case. However, sometimes the expectation can be directly computed or estimated w.r.t. the loss of the best expert in the pool. For example, see  approaches of \cite{VoV2001,Kaln2007,zhdanov2010identity} applied to Online Kernel Regression.

The regret bound \ref{cor-regret-bound} is a linear function of game length $T$. Nevertheless, if the game length $T$ is known in advance, one may achieve sub-linear regret bound by choosing the learning rate $\eta$ to be dependent on game length $T$. Particular examples of learning rates $\eta=\eta(T)$ for specific underlying distributions $p(\cdot)$ are provided in following Section \ref{sec-examples}.

\section{Examples}
\label{sec-examples}

In this Section we provide the examples of useful underlying probability models $p(\cdot)$ and use them to apply the algorithm $\mathcal{G}$ to construct online expert weight allocation algorithms. We consider a finite pool of experts ${\mathcal{N}=\{1,2,\dots, N\}}$.

\subsection{Basic Delayed Exponential Weights (Hedge)}
\label{hedge-delayed}
Consider the following underlying probability $p$. Let ${p(n_{1})=p_{0}(n_{1})}$ be some prior and ${p(n_{t}|n_{t-1})=\mathbb{I}_{[n_{t}=n_{t-1}]}}$ for ${t=2,\dots,T}$. This means that the hidden active expert does not change during the game. We denote the corresponding algorithm applied to $p$ by ${\mathcal{G}_{\text{base}}=\mathcal{G}_{\text{base}}(p_{0})}$. The corresponding graphical model is shown in Figure \ref{figure:model-simple-n}.

\begin{figure}[!htb]
\begin{center}
\includegraphics[scale=0.6]{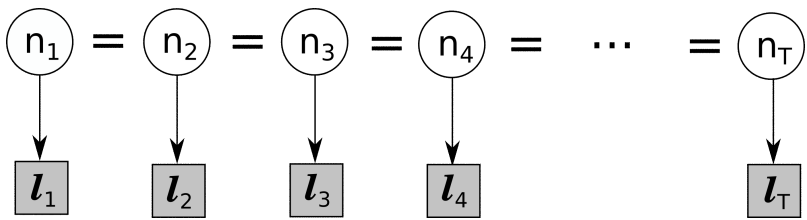}
\end{center}
\caption{Hedge probabilistic model for data generation process}
\label{figure:model-simple-n}
\end{figure} 
It is easy to see that for all $t$ the weight allocation $${w_{t}^{n}\propto p_{0}(n)\cdot e^{-\eta L_{\mathcal{D}_{t}}^{n}}}$$ is proportional to the observed losses of the expert $n\in\mathcal{N}$. If there are no delays ($D_{t}\equiv 0$ for all $t$), then the algorithm becomes classical Hedge by \cite{FrS97}.

\subsubsection{Algorithm}

The pseudo-code of Algorithm \ref{algorithm-delayed-hedge} ($\mathcal{G}_{\text{base}}$) is shown below. In the code we assume that the operation \textbf{Output}($\ldots$) sets the weight allocation ($\bm{w}_{t}$) for the current step. Function \textbf{GetRevealedLosses}() obtains all the vectors of losses ${\bm{l}_{\tau}=(l_{\tau}^{1},\dots,l_{\tau}^{N})}$ of the steps ${\tau\in d\mathcal{D}_{t}}$ in the form of an iterable list of pairs ${(\tau, \bm{l}_{\tau})}$.

\begin{algorithm}
\SetKwInOut{Parameters}{Parameters}
\Parameters{Number of experts $N$, rounds $T$;\\ Learning rate $\eta>0$; Prior $\bm{p}_{0}\in\Delta(N)$.}
$\bm{w}\leftarrow \bm{p}_{0}$\;
\For{$t=1,2,\dots,T$}{
	\textit{Output}($\bm{w}$)\;
    \For{$(\tau, \bm{l}_{\tau})$ in GetRevealedLosses()}{
    	$\bm{w}\leftarrow \bm{w} \cdot \exp(-\eta\cdot \bm{l}_{\tau})$\;
        $\bm{w}\leftarrow \bm{w}/\|\bm{w}\|_{1}$.
    }  
 }
\caption{Delayed Hedge ($\mathcal{G}_{base}$)}
\label{algorithm-delayed-hedge}
\end{algorithm}

The algorithm requires $O(N)$ \textbf{memory} and $O(NT)$ \textbf{time complexity}.

\subsubsection{Regret bound}

According to Corollary \ref{corollary-regret-sequence}, the regret of the algorithm with respect to any fixed expert $n\in\mathcal{N}$ is bounded:
\begin{eqnarray}
R_{T}(n)=H_{T}-L_{T}^{n}\leq
-\frac{1}{\eta}\ln p_{0}(n)+\eta\frac{H^{2}}{8}T+\eta \big[\frac{H^{2}\cdot \sum_{t=1}^{T}D_{t}}{4}\big].
\nonumber
\end{eqnarray}

The typical prior is $p_{0}\equiv \frac{1}{N}$. For this basic case in the non-delayed feedback setting ($D_{t}\equiv 0$) the $\eta$ is chosen in advance (with prior knowledge of $T$) to minimize the regret. The optimal choice is $\eta\propto \frac{1}{H\sqrt{T}}$, which results in $O(\sqrt{T})$ classical regret.

However, the choice of optimal $\eta$ in the delayed setting highly depends on how the sequence of delays is generated. If the learner knows $\sum_{t=1}^{T}D_{t}$ in advance or $D_{t}$ is sampled from some distribution with known expectation $\mathbb{E}D$, the optimal choice is 
\begin{equation}
\eta\propto \frac{1}{H\sqrt{T+\sum_{t=1}^{T}D_{t}}}
\qquad
\text{or}
\qquad
\eta\propto \frac{1}{H\sqrt{T(1+\mathbb{E}D)}} 
\nonumber
\end{equation}
respectively. This choice results in $O(\sqrt{T+\sum_{t=1}^{T}D_{t}})$ and $O(\sqrt{T(1+\mathbb{E}D)})$ regret bounds respectively.

If the sequence of delays is chosen by an adversary, the classical choice $\eta\propto \frac{1}{H\sqrt{T}}$ results in $O[\sqrt{T}(1+\overline{D})]$ regret, where $\overline{D}=\frac{1}{T}\sum_{t=1}^{T}D_{t}$.

\subsection{Adaptive Delayed Exponential Weights (Fixed Share)}
\label{sec-fs}

Consider the following underlying probability $p$. Let ${p(n_{1})=p_{0}(n_{1})}$ be some prior and
\begin{equation}
{p(n_{t}|n_{t-1})=\alpha_{t}p_{0}(n_{t})+(1-\alpha_{t})\cdot \mathbb{I}_{[n_{t}=n_{t-1}]}}
\label{fixed-share-trans-prob}
\end{equation}
for $t=2,\dots,T$ and sequence $0\leq \alpha_{2},\dots,\alpha_{T}\leq 1$. This means that the hidden active expert changes to random (according to prior $p_{0}$) between steps $t-1$ and $t$ with some small probability $\alpha_{t}$. 

We denote the corresponding algorithm applied to $p$ by $\mathcal{G}_{\text{fs}}$. The graphical model is shown in Figure \ref{figure:model-markov-n}.

\begin{figure}[!htb]
\begin{center}
\includegraphics[scale=0.6]{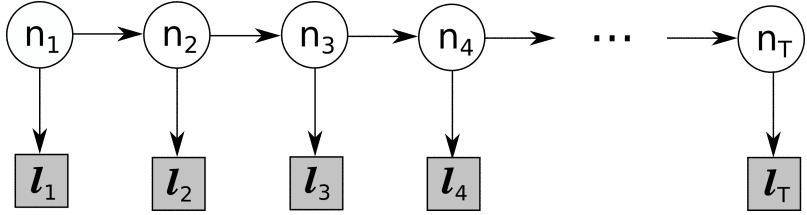}
\end{center}
\caption{Fixed Share probabilistic model for data generation process}
\label{figure:model-markov-n}
\end{figure}

The sequence $\alpha_{t}$ can be arbitrary. However, the classical approach is to use $\alpha_{t}=\frac{1}{t}$ (see \cite{HeW98,AdaAda16,cesa2012mirror}), because in this special case the regret bound is better (than e.g. in the case $\alpha_{t}\equiv const$). In our case at the end of the subsection we will also use the sequence $\alpha_{t}=\frac{1}{t}$ when estimating the regret.

\subsubsection{Equivalence to Fixed Share in the Non-Delayed Setting}

To begin with, we examine the application of algorithm $\mathcal{G}$ to the described probabilistic model $p(\cdot)$ in the non-delayed case, i.e., $D_{t}\equiv 0$ for all ${t=1,2,\dots,T}$. In the non-delayed case we have $\mathcal{D}_{t}=\{1,2,\dots,t\}$ for all $t$. Thus, for all $t$ we get
$$w_{t}^{n}=p(n_{t}=n|\bm{L}_{\mathcal{D}_{t-1}})=p(n_{t}=n|\bm{L}_{t-1}).$$
We set ${\bm{u}_{t}=(u_{t}^{1},\dots, u_{t}^{N})\in\Delta(\mathcal{N})}$,  $u_{t}^{n}=p(n_{t}=n|\bm{L}_{t})$ for all ${n\in\mathcal{N}}$ and ${t=1,2,\dots, T}$. We get
\begin{eqnarray}
w_{t}^{n}=p(n_{t}=n|\bm{L}_{t-1})=
\nonumber
\\
\sum_{n'\in\mathcal{N}}\bigg[p(n_{t}=n|n_{t-1}=n')\cdot \underbrace{p(n_{t-1}=n'|\bm{L}_{t-1})}_{u_{t-1}^{n'}}\bigg].
\label{fs-reweighing-part}
\end{eqnarray}
Combining \eqref{fs-reweighing-part} with \eqref{fixed-share-trans-prob} we see that \begin{equation}
\bm{w}_{t}=(1-\alpha_{t})\cdot \bm{u}_{t-1}+\alpha_{t}\cdot \bm{p}_{0}.
\label{fixed-share-first-update}
\end{equation}
On the other hand,
\begin{eqnarray}
u_{t}^{n}=p(n_{t}=n|\bm{L}_{t})=\frac{p(\bm{L}_{t}|n_{t}=n)\cdot p(n_{t}=n)}{p(\bm{L}_{t})}=
\nonumber
\\
\frac{p(\bm{l}_{t}|n_{t}=n)\cdot p(\bm{L}_{t-1}|n_{t}=n)\cdot p(n_{t}=n)}{p(\bm{L}_{t})}=
\nonumber
\\
\underbrace{p(n_{t}=n|\bm{L}_{t-1})}_{w_{t}^{n}}\cdot \underbrace{p(\bm{l}_{t}|n_{t}=n)}_{\exp (-\eta\cdot l_{t}^{n})}\cdot \big[\frac{p(\bm{L}_{t-1})}{p(\bm{L}_{t})}\big].
\nonumber
\end{eqnarray}
Thus, 
\begin{equation}
\bm{u}_{t}\propto \bm{w}_{t}\cdot \exp(-\eta \cdot \bm{l}_{t}).
\label{fixed-share-second-update}
\end{equation}
Formulas \eqref{fixed-share-first-update} and \eqref{fixed-share-second-update} mean that the algorithm's decision $\bm{w}_{t}$ can be iteratively updated step by step by using the additional weight $\bm{u}_{t}$. The obtained weight updates \eqref{fixed-share-first-update} and \eqref{fixed-share-second-update} exactly match the updates of the Fixed Share algorithm by \cite{HeW98}. Thus, in the non-delayed case $\mathcal{G}(p)$ is equal to Fixed Share.

\subsubsection{Algorithm for the Delayed Setting}

Now we examine the algorithm $\mathcal{G}(p)$ under the setting of the delayed feedback, i.e., for all ${t=1,2,\dots,T}$ delay $D_{t}$ is some non-negative integer value.

For all $t=1,2,\dots,T$ and $\tau\leq t$ we use $\mathcal{D}_{t}^{\tau}$ to denote the set of all time steps $t'\leq \tau$ such that the loss vector $\bm{l}_{t'}$ is revealed not later than the step $t$. Formally, we define
$$\mathcal{D}_{t}^{\tau}=\{t'|(t'+D_{t'}\leq t)\land (t'\leq \tau)\}.$$

In the next few paragraphs we describe the efficient scheme to recompute the algorithms decision $\bm{w}_{t}$ at every step $t$.

Suppose that at the beginning of the step $t$ we keep all the probabilities $p(n_{\tau}|\bm{L}_{\mathcal{D}_{t}^{\tau}})$ for all $\tau=1,\dots,t$ and $n\in\mathcal{N}$. We denote the corresponding $N$-dimensional probability vectors by $\bm{u}_{\tau}$. We also denote $\bm{u}_{0}=\bm{p}_{0}$.  Similar to \eqref{fs-reweighing-part} calculations lead to the simple formula that allows to obtain $\bm{w}_{t}$:
$$\bm{w}_{t}=(1-\alpha_{t})\cdot \bm{u}_{t-1}+\alpha_{t}\cdot \bm{p}_{0}.$$
After the decision on $\bm{w}_{t}$ is made, the algorithm obtains losses of steps ${\tau\in d\mathcal{D}_{t}}$. Thus, we need to calculate new probability vector $\bm{u}_{t}$ with coordinates $p(n_{\tau}|\bm{L}_{\mathcal{D}_{t}^{\tau}})$. Moreover, we have to update all vectors $\bm{u}_{\tau}$ for $\tau<t$ from $p(n_{\tau}|\bm{L}_{\mathcal{D}_{t-1}^{\tau}})$ to $p(n_{\tau}|\bm{L}_{\mathcal{D}_{t}^{\tau}})$.

Let $\tau_{\min}=\min\{\tau:\tau \in d\mathcal{D}_{t}\}$. Note that all  $\bm{u}_{\tau}$ for $\tau <\tau_{\min}$ do not require being updated because $\mathcal{D}_{t}^{\tau}=\mathcal{D}_{t-1}^{\tau}$. Next, for $\tau=\tau_{\min},\dots,t-1, t$ we recompute the vectors $\bm{u}_{\tau}$ iteratively. 

We explain how to compute $\bm{u}_{\tau}$ below (assuming that previous $\bm{u}_{\tau-1}$ is already computed). For convenience, we introduce the temporary vector variable ${\bm{v}=(v^{1},\dots, v^{N})\in\Delta(\mathcal{N})}$, where $v^{n}=p(n_{\tau}=n|\bm{L}_{\mathcal{D}_{t}^{\tau-1}})$ for all ${n\in\mathcal{N}}$. First, we express $\bm{v}$ using $\bm{u}_{\tau-1}$. Next, we express $\bm{u}_{\tau}$ using $\bm{v}$.

We deduce the formula to compute $v^{n_{\tau}}$ by using $u_{\tau-1}^{n_{\tau}}$:
\begin{eqnarray}
v^{n_{\tau}}=p(n_{\tau}|\bm{L}_{\mathcal{D}_{t}^{\tau-1}})=\sum_{n_{\tau-1}\in\mathcal{N}}p(n_{\tau-1}|\bm{L}_{\mathcal{D}_{t}^{\tau-1}})\cdot p(n_{\tau}|n_{\tau-1})=
\label{fixed-share-transition-decomposition}
\\
\sum_{n_{\tau-1}\in\mathcal{N}}\bigg[p(n_{\tau-1}|\bm{L}_{\mathcal{D}_{t}^{\tau-1}})\cdot \big[\alpha_{\tau}\cdot p_{0}(n_{\tau})+(1-\alpha_{\tau})\cdot \mathbb{I}_{[n_{\tau}=n_{\tau-1}]}\big]\bigg]=
\label{fixed-share-transition-usage}
\\
\underbrace{\big[\sum_{n_{\tau-1}\in\mathcal{N}}p(n_{\tau-1}|\bm{L}_{\mathcal{D}_{t}^{\tau-1}})\big]}_{\text{Sums to }1}\cdot \alpha_{\tau}\cdot p_{0}(n_{\tau})+(1-\alpha_{\tau})\cdot u_{\tau-1}^{n_{\tau}}=
\nonumber
\\
\alpha_{\tau}\cdot p_{0}(n_{\tau})+(1-\alpha_{\tau})\cdot u_{\tau-1}^{n_{\tau}}.
\label{fixed-share-transition-recomputing}
\end{eqnarray}
In line \eqref{fixed-share-transition-decomposition} we exploit the fact that the elements of $\mathcal{D}_{t}^{\tau-1}$ are strictly lower than $\tau$. In line \eqref{fixed-share-transition-usage} we use the definition \eqref{fixed-share-trans-prob} of the transition probability. The vector form of \eqref{fixed-share-transition-recomputing} is
\begin{equation}\bm{v}=(1-\alpha_{\tau})\cdot \bm{u}_{\tau-1}+\alpha_{\tau}\cdot \bm{p}_{0}.
\nonumber
\end{equation}

To derive $\bm{u}_{\tau}$ using $\bm{v}$ we consider two cases: $\tau\notin\mathcal{D}_{t}$ and $\tau\in\mathcal{D}_{t}$. In the first case  ${\bm{L}_{\mathcal{D}_{t}^{\tau-1}}\equiv\bm{L}_{\mathcal{D}_{t}^{\tau}}}$, which leads to $\bm{u}_{\tau}=\bm{v}$. If $\tau\in\mathcal{D}_{t}$, we have
\begin{eqnarray}
u_{\tau}^{n_{\tau}}=p(n_{\tau}|\bm{L}_{\mathcal{D}_{t}^{\tau}})\propto
p(\bm{L}_{\mathcal{D}_{t}^{\tau}}|n_{\tau})\cdot p(n_{\tau})=
\nonumber
\\
p(\bm{L}_{\mathcal{D}_{t}^{\tau-1}}|n_{\tau})\cdot p(\bm{l}_{\tau}|n_{\tau})\cdot p(n_{\tau})\propto p(n_{\tau}|\bm{L}_{\mathcal{D}_{t}^{\tau-1}})\cdot p(\bm{l}_{\tau}|n_{\tau})\propto
\nonumber
\\
p(n_{\tau}|\bm{L}_{\mathcal{D}_{t}^{\tau-1}})\cdot \exp(-\eta l_{\tau}^{n_{\tau}})=v^{n_{\tau}}\cdot \exp(-\eta l_{\tau}^{n_{\tau}}).
\label{fixed-share-update-observation}
\end{eqnarray}
The vector form of expression \eqref{fixed-share-update-observation} is $$\bm{u}_{\tau}\propto \bm{v}\cdot \exp(-\eta \cdot \bm{l}_{\tau}).$$

The pseudo-code of algorithm \ref{algorithm-delayed-fixed-share} ($\mathcal{G}_{\text{fs}}$) is shown below.  In addition to the notations of Algorithm \ref{algorithm-delayed-hedge} ($\mathcal{G}_{\text{base}}$), we assume that an extra function \textbf{GetSwitchProbability}() provides the value of the current switch probability $0\leq\alpha_{t}\leq 1$ (which may be chosen online).
 
\begin{algorithm}[htb!]
\SetKwInOut{Parameters}{Parameters}
\Parameters{Number of experts $N$, rounds $T$;\\ Learning rate $\eta>0$; Prior $\bm{p}_{0}\in\Delta(N)$.}
$\bm{u}\leftarrow$ List($[\bm{p}_{0}]$)\;
$\bm{\alpha}\leftarrow$ List($[1]$)\;
$\bm{l}\leftarrow$ List([Null])\;
$\bm{v}\leftarrow$ Null\;
\For{$t=1,2,\dots,T$}{
	$\bm{\alpha}.$append$\big($\textit{GetSwitchProbability}()$\big)$\;
    $\bm{u}$.append$\big((1-\bm{\alpha}[t])\cdot \bm{u}[t-1] + \bm{\alpha}[t]\cdot \bm{p}_{0}\big)$\; 
	$\bm{l}.$append(Null)\;
	\textit{Output}($\bm{u}[t]$)\;
    $\tau_{\min}\leftarrow\infty$\;
    \For{$(\tau, \bm{l}_{\tau})$ in GetRevealedLosses()}{
    	$\bm{l}[\tau]\leftarrow\bm{l}_{\tau}$\;
        $\tau_{\min}\leftarrow\min(\tau, \tau_{\min})$
    }
    \For{$\tau\in [\tau_{\min}; t]$}{
    	$\bm{v}\leftarrow (1-\bm{\alpha}[\tau])\cdot \bm{u}[\tau-1]+\bm{\alpha}[\tau]\cdot \bm{p}_{0}$\;
      \uIf{$\bm{l}[\tau]\neq$ Null}{
      	$\bm{u}[\tau]\leftarrow \bm{v}\cdot\exp(-\eta\bm{l}[\tau])$\;
        $\bm{u}[\tau]\leftarrow \bm{u}[\tau]/\|\bm{u}[\tau]\|_{1}$\;
        }
      \Else{
        $\bm{u}[\tau]\leftarrow \bm{v}$ \;
      }
    }
 }
\caption{Delayed Fixed Share ($\mathcal{G}_{\text{fs}}$)}
\label{algorithm-delayed-fixed-share}
\end{algorithm}

The \textbf{List}() class corresponds to the dynamic array. We assume that it has integer index $\{0,1,\dots,|List|-1\}$, supports the append-to-right operation in $O(1)$ time. We also assume that all operations to get or set list element (by index) require $O(1)$ time. At the end of each step $t$ list $\bm{u}$ keeps the posterior probabilities described above.

The \textbf{time complexity} of the algorithm is bounded by $${O\big(N\cdot(T+\sum_{t=1}^{T}D_{t})\big)}.$$ Indeed, at the steps $t$ such that $|d\mathcal{D}_{t}|=0$ the algorithm performs $O(N)$ operations. At other steps the algorithm performs 
$${O\big(N\cdot(t+1-\min\{\tau:\tau \in d\mathcal{D}_{t}\})\big)}$$ operations which are bounded by $O(N\cdot D_{\tau})$ for the minimal $\tau\in d\mathcal{D}_{t}$.

The \textbf{memory complexity} of the algorithm is $O(NT)$. However, it is possible to significantly reduce the memory complexity. Note that if for some $\tau, t$ we have $\mathcal{D}_{t}^{\tau}=\{1,2,\dots,t\}$, the weights $\bm{u}_{0},\dots,\bm{u}_{\tau-1}$ will never be used or recomputed after the step $t$. Thus, they become useless, and it is meaningful to keep only elements $\bm{u}_{t'}$ with $t'\geq \tau$ (same for lists $\bm{l}$ and $\bm{\alpha}$). The reduction will result in $O(N\cdot \max D_{t})$ memory complexity. We did not include the explained trick in the pseudo-code of Algorithm \ref{algorithm-delayed-fixed-share} in order to keep it simple.

\subsubsection{Regret Bound}

We use $\alpha_{t}=\frac{1}{t}$. We combine Corollary \ref{corollary-regret-sequence} with Lemma \ref{fixed-share-prob-bound} and obtain the regret bound for the algorithm with respect to any switching sequence ${{N}_{T}=(n_{1},n_{2},\dots,n_{T})}$:
\begin{eqnarray}
R_{T}({N}_{T})\leq H_{T}-L_{T}^{{N}_{T}}\leq 
\nonumber
\\
(K+1)\cdot \frac{\ln N+\ln T}{\eta}+\eta\frac{H^{2}}{8}T+
\eta \big[\frac{H^{2}\cdot \sum_{t=1}^{T}D_{t}}{4}\big],
\label{fixed-share-regret-bound}
\end{eqnarray}
where $K=|\{t:\,n_{t}\neq n_{t-1}\}|$ is the number of expert's switches in ${N}_{T}$.

Similar to the non-adaptive case, the algorithm requires choosing optimal learning rate $\eta$ in order to minimize the regret bound. The optimal $\eta$ should be chosen with respect to $T$ and $\sum_{t=1}^{T}D_{t}$.\footnote{It is also possible to minimize the bound w.r.t. particular number of switches $K$.} The following discussion is similar to the one at the end of the previous subsection \ref{hedge-delayed}.

If the learner knows $\sum_{t=1}^{T}D_{t}$ beforehand or $D_{t}$ is sampled from some distribution with known expectation $\mathbb{E}D$, the choice of 
\begin{equation}
\eta\propto \frac{1}{H}\sqrt{\frac{\ln T}{T+\sum_{t=1}^{T}D_{t}}}\
\quad
\text{or}
\quad
\eta\propto \frac{1}{H}\sqrt{\frac{\ln T}{T(1+\mathbb{E}D)}}
\nonumber
\end{equation}
respectively results in
$${O\big((K+2)\cdot \sqrt{(T+\sum_{t=1}^{T}D_{t})\cdot \ln T}\big)}$$ and $${O\big((K+2)\cdot \sqrt{T(1+\mathbb{E}D)\cdot \ln T}\big)}$$ (expected) regret bound with respect to any sequence with no more than $K$ expert switches.

If the sequence of delays is chosen by an adversary and unknown to the learner, then classical choice $\eta\propto \frac{\sqrt{\ln T}}{H\sqrt{T}}$ results in ${O[(K+2)\sqrt{T\ln T}(1+\overline{D})]}$ regret, where ${\overline{D}=\frac{1}{T}\sum_{t=1}^{T}D_{t}}$.

\section{Proof of Performance}
\label{sec-proofs}
In this section we prove Theorem \ref{theorem-main-loss-bound-D}. The proof is complicated, and we split it into two sequential parts. Firstly, we prove the bound \eqref{main-loss-bound-convex} for the non-delayed case in Subsection \ref{sec-proof-1}, i.e., $\{D_{t}\}_{t=1}^{T}=(0,\dots, 0)$. Secondly, we obtain the bound \eqref{main-loss-bound-convex} for arbitrary sequence of delays $\{D_{t}\}_{t=1}^{T}$ in Subsection \ref{sec-proof-2}.

\subsection{Bound for Non-delayed Setting}
\label{sec-proof-1}

We set $D_{t}\equiv 0$ for all $t$ and deal with the bound for algorithm $\mathcal{G}$ in this case. Note that $\mathcal{D}_{t}=\{1,2,\dots,t\}$ for all $t=1,2,\dots,T$ and $\bm{L}_{\mathcal{D}_{t}}=\bm{L}_{t}$.

\begin{proof}Recall that $w_{t}^{n_{t}}=p(n_{t}|\bm{L}_{\mathcal{D}_{t-1}})=p(n_{t}|\bm{L}_{t-1})$. Define the \textbf{mixloss} at the step $t$:
\begin{eqnarray}
m_{t}=-\frac{1}{\eta}\ln\big[\sum_{n_{t}\in\mathcal{N}}e^{-\eta l_{t}^{n_{t}}}\cdot w_{t}^{ n_{t}}\big]=
-\frac{1}{\eta}\ln\big[\sum_{n_{t}\in\mathcal{N}}e^{-\eta l_{t}^{n_{t}}}\cdot p( n_{t}|\bm{L}_{t-1})\big]=
\label{mixloss-definition}
\\
-\frac{1}{\eta}\ln\big[\sum_{n_{t}\in\mathcal{N}}Z\cdot p(\bm{l}_{t}|n_{t})\cdot p(n_{t}|\bm{L}_{t-1})\big]=
-\frac{1}{\eta}\ln Z-\frac{1}{\eta}\ln p(\bm{l}_{t}|\bm{L}_{t-1}).
\nonumber
\end{eqnarray}
Define the \textbf{cumulative mixloss} $M_{T}$ over the entire game:
\begin{eqnarray}
M_{T}=\sum_{t=1}^{T}m_{t}=-\frac{T}{\eta}\ln Z-\frac{1}{\eta}\ln\prod_{t=1}^{T}p(\bm{l}_{t}|\bm{L}_{t-1})=
\nonumber
\\
-\frac{T}{\eta}\ln Z-\frac{1}{\eta}\ln p(\bm{L}_{T})=-\frac{T}{\eta}\ln Z-\frac{1}{\eta}\ln \bigg[\sum_{N_{T}\in\mathcal{N}^{T}}p(N_{T})p(\bm{L}_{T}|N_{T})\bigg]=
\nonumber
\\
-\frac{T}{\eta}\ln Z-\frac{1}{\eta}\ln \bigg[\sum_{N_{T}\in\mathcal{N}^{T}}\big[p(N_{T})\prod_{t=1}^{T}p(\bm{l}_{t}|n_{t})\big]\bigg]=
\nonumber
\\
-\frac{1}{\eta}\ln\bigg[ \sum_{N_{T}\in\mathcal{N}^{T}}\big[p(N_{T})\prod_{t=1}^{T}\big(\underbrace{Z\cdot p(\bm{l}_{t}|n_{t})}_{e^{-\eta l_{t}^{{n}_{t}}}}\big)\big]\bigg]=
-\frac{1}{\eta}\ln \bigg[\mathbb{E}_{p({N}_{T})}\big[e^{-\eta L_{T}^{{N}_{T}}}\big]\bigg].
\nonumber
\end{eqnarray}
For all $t=1,\dots,T$ we apply Hoeffding's inequality \eqref{hoeffding} to a random variable 
$$X_{t}=l_{t}^{n_{t}}\in [0,H],$$ 
where $n_{t}\sim p(n_{t}|\bm{L}_{t-1})=w_{t}^{n_{t}}$:
$$\ln \sum_{n=1}^{N}w_{t}^{n}e^{-\eta l_{t}^{n_{t}}}\leq -\eta \langle \bm{w}_{t},\bm{l}_{t}\rangle+\eta^{2}\frac{H^{2}}{8},$$
which is equal to
\begin{equation}
h_{t}\leq m_{t}+\eta \frac{H^2}{8}.
\label{single-loss-bound}
\end{equation}
We sum \eqref{single-loss-bound} for $t=1,2,\dots,T$ and obtain
\begin{eqnarray}H_{T}\leq M_{T}+\eta \frac{H^2}{8}T=
-\frac{1}{\eta}\ln \bigg[\mathbb{E}_{p({N}_{T})}\big[e^{-\eta L_{T}^{{N}_{T}}}\big]\bigg]+\eta \frac{H^2}{8}T,
\label{non-delayed-loss-bound}
\end{eqnarray}
which finishes the proof.\end{proof}

\subsection{Bound for Delayed Setting}
\label{sec-proof-2}

In this section we consider the case of arbitrary sequence of delays $\{D_{t}\}_{t=1}^{T}$.

\begin{proof} We use the superscript $(\ldots)^{\mathcal{D}}$ to denote the variables obtained by algorithm $\mathcal{G}$ (for example, weights $\bm{w}_{t}^{\mathcal{D}}$, etc.) with the sequence of delays $\{D_{t}\}_{t=1}^{T}$. Our main idea is to prove that the weights $\bm{w}_{t}^{\mathcal{D}}$ are approximately equal to the weights $\bm{w}_{t}^{0}$ obtained by the algorithm in the game with the same experts but with no delays, i.e., $\{D_{t}\}_{t=1}^{T}=(0,\dots, 0)$. Thus, the losses $h_{t}^{\mathcal{D}}$ and $h_{t}^{0}$ will be approximately equal.

We divide this part of the proof of the theorem into two steps:

\vspace{2mm}

\noindent \textbf{Step 1. Proof for a simple probability distribution $p$}
\vspace{1mm}

To begin with, we consider the case of a simple Hidden Markov Model $p(\cdot)$. Let ${p(n_{1})=p_{0}(n_{1})}$ and ${p(n_{t+1}|n_{t})=\mathbb{I}_{[n_{t+1}=n_{t}]}}$ for all ${t=1,\dots,T-1}$. The corresponding algorithm is $\mathcal{G}_{\text{base}}=\mathcal{G}(p)$.

We compare the losses $H_{T}^{0}$ and $H_{T}^{\mathcal{D}}$ of algorithm $\mathcal{G}_{\text{base}}$ applied to the same data with no delays and with the given sequence of delays $\{D_{t}\}_{t=1}^{T}$ respectively.
\begin{eqnarray}
|H_{T}^{0}-H_{T}^{\mathcal{D}}|=|\sum_{t=1}^{T}h_{t}^{0}-\sum_{t=1}^{T}h_{t}^{\mathcal{D}}|\leq \sum_{t=1}^{T}|h_{t}^{0}-h_{t}^{\mathcal{D}}|= 
\nonumber
\\
\sum_{t=1}^{T}|\langle \bm{w}_{t}^{0}, \bm{l}_{t}\rangle - \langle \bm{w}_{t}^{\mathcal{D}}, \bm{l}_{t}\rangle |= \sum_{t=1}^{T}|\langle \bm{w}_{t}^{0}- \bm{w}_{t}^{\mathcal{D}}, \bm{l}_{t}\rangle | \leq 
\nonumber
\\
H\cdot \sum_{t=1}^{T}\max_{\mathcal{N}'\subset \mathcal{N}}\bigg[\sum_{n\in \mathcal{N}'}\big[(w_{t}^{n})^{0}-(w_{t}^{n})^{\mathcal{D}}\big]\bigg]
\label{unfinished-loss-diff-bound}
\end{eqnarray}

Note that ${(w_{t}^{n})^{0}\propto e^{-\eta L_{t-1}^{n}}}$ and ${(w_{t}^{n})^{\mathcal{D}}\propto e^{-\eta L_{\mathcal{D}_{t-1}}^{n}}}$ for all $t$. This means that ${(w_{t}^{n})^{0}\propto [(w_{t}^{n})^{\mathcal{D}}\cdot a^{n}]}$, where 
$$-\frac{1}{\eta}\ln (a^{n})=\sum_{\tau=1}^{t-1}l_{\tau}^{n}-\sum_{\tau\in\mathcal{D}_{t-1}}l_{\tau}^{n}\in \big[0, [(t-1)-|\mathcal{D}_{t-1}|]\cdot H\big].$$
Thus, according to Lemma \ref{lemma-change-bound}, we obtain the bound
$$\max_{\mathcal{N}'\subset \mathcal{N}}\bigg[\sum_{n\in \mathcal{N}'}\big[(w_{t}^{n})^{0}-(w_{t}^{n})^{\mathcal{D}}\big]\bigg]\leq \eta H\frac{t-1-|\mathcal{D}_{t-1}|}{4}$$
for all $t$. Combining it with \eqref{unfinished-loss-diff-bound} and Lemma \ref{lemma-delay-sum} we obtain:
$$|H_{T}^{0}-H_{T}^{\mathcal{D}}|\leq \eta\frac{H^{2}}{4}\big[\frac{T(T-1)}{2}-\sum_{t=1}^{T-1}|\mathcal{D}_{t}|\big]=\eta\frac{H^{2}}{4}\sum_{t=1}^{T}D_{t}.$$
The final step is to combine current result with the loss bound \eqref{non-delayed-loss-bound} for the non-delayed case:
\begin{eqnarray}
H_{T}^{\mathcal{D}}\leq H_{T}^{0}+|H_{T}^{0}-H_{T}^{\mathcal{D}}|\leq
\nonumber
\\
 -\frac{1}{\eta}\ln \bigg[\mathbb{E}_{p({N}_{T})}\big[e^{-\eta L_{T}^{{N}_{T}}}\big]\bigg]+\eta\frac{H^{2}}{8}T+\eta\frac{H^{2}}{4}\sum_{t=1}^{T}D_{t},
 \nonumber
\end{eqnarray}
and finish the proof of the bound for algorithm $\mathcal{G}_{\text{base}}$.

\vspace{2mm}

\noindent \textbf{Step 2. Proof for an arbitrary probability distribution $p$}
\vspace{1mm}

Now we consider the case of an arbitrary probability distribution $p$. From the given set of experts $\mathcal{N}$ we create a new super set $\mathcal{S}=\mathcal{N}^{T}$ of \textbf{super experts} $s\in\mathcal{S}$ ($\mathcal{S}$ for \textbf{S}uper). Each super expert $s$ corresponds to some sequence ${{N}_{T}=(n_{1},\dots,n_{T})\in \mathcal{N}^{T}}$ of basic experts $n\in\mathcal{N}$ of length $T$. We denote the $t$-th component of super expert $s$ by $n_{t}(s)$. We denote the full sequence of experts corresponding to $s$ by ${N}_{T}(s)$. We do not use subscript in order not to overburden the notation. The loss of super expert $s\in\mathcal{S}$ at the step $t$ is $l_{t}^{n_{t}(s)}$, where $l_{t}^{n}$ (for $n\in\mathcal{N}$) are the losses of basic experts. We use $E(\bm{L}_{\mathcal{D}_t})$ and $E(\bm{l}_{t})$ to denote all the super experts' losses at the steps $\mathcal{D}_{t}$ and $t$ respectively ($E$ for \textbf{E}nhanced).

We define the probability model for hidden super experts. In order not to confuse the reader with notation, we use capital $P$ (instead of regular $p$) to denote all probabilities related to super experts. Let $P(s_{1})=P_{0}(s_{1})=p({N}_{T}(s_{1}))$ and 
$$P(s_{t+1}|s_{t})=[s_{t+1}=s_{t}].$$
The described probability distribution corresponds to algorithm $\mathcal{G}_{\text{base}}(P)$ for super experts $s\in\mathcal{S}$ and initial distribution $P_{0}$. We have $s_{1}=s_{2}=\dots=s_{T}$ w.p. 1.

The main idea is to show that the losses of algorithm $\mathcal{G}_{\text{base}}$ are equal to the losses of algorithm $\mathcal{G}(p)$. In order to prove this, we show that for all $t$ the sum of the weights 
$$\widehat{w}_{t}^{n}\stackrel{\mbox{def}}{=}\sum_{s|n_{t}(s)=n}w_{t}^{s}$$
in algorithm $\mathcal{G}_{\text{base}}(P)$ is equal to $w_{t}^{n}$ in algorithm $\mathcal{G}(p)$. This sum corresponds to the weight that is allocated to the base expert $n\in\mathcal{N}$ as a part of the super experts' weight allocation for step $t$. We perform several calculations:

\begin{eqnarray}
\widehat{w}_{t}^{n}=\sum_{s|n_{t}(s)=n}w_{t}^{s}=
\nonumber
\\
\sum_{s|n_{t}(s)=n}P\big(s|E(\bm{L}_{\mathcal{D}_{t-1}})\big)=
\sum_{s|n_{t}(s)=n}\frac{P\big(E(\bm{L}_{\mathcal{D}_{t-1}}|s)\big)P(s)}{P\big(E(\bm{L}_{\mathcal{D}_{t-1}})\big)}=
\nonumber
\end{eqnarray}
Now note that $P(s)=P_{0}(s)=p\big({N}_{T}(s)\big),$ and 
\begin{eqnarray}
P\big(E(\bm{L}_{\mathcal{D}_{t-1}})|s)=\prod_{\tau\in\mathcal{D}_{t-1}}P(E(\bm{l}_{\tau})|s)=
\nonumber
\\
\prod_{\tau\in\mathcal{D}_{t-1}}p\big(\bm{l}_{\tau}|n_{\tau}(s)\big)=
p\big(\bm{L}_{\mathcal{D}_{t-1}}|{N}_{T}(s)\big)=
p\big(\bm{L}_{\mathcal{D}_{t-1}}|{N}_{\mathcal{D}_{t-1}}(s)\big).
\nonumber
\end{eqnarray}
Thus, we continue computations:
\begin{eqnarray}
\sum_{s|n_{t}(s)=n}w_{t}^{s}=
\sum_{s|n_{t}(s)=n}\frac{p\big(\bm{L}_{\mathcal{D}_{t-1}}|{N}_{\mathcal{D}_{t-1}}(s)\big)\cdot p\big({N}_{T}(s)\big)}{P\big(E(\bm{L}_{\mathcal{D}_{t-1}})\big)}=
\nonumber
\\
\sum_{{N}_{T}|n_{t}=n}\frac{p(\bm{L}_{\mathcal{D}_{t-1}}|{N}_{\mathcal{D}_{t-1}})\cdot p({N}_{T})}{P\big(E(\bm{L}_{\mathcal{D}_{t-1}})\big)}=
\nonumber
\\
\sum_{{N}_{T}|n_{t}=n}\frac{p(\bm{L}_{\mathcal{D}_{t-1}}|{N}_{\mathcal{D}_{t-1}})\cdot p({N}_{\mathcal{D}_{t-1}})\cdot p({N}_{T}|{N}_{\mathcal{D}_{t-1}})}{P\big(E(\bm{L}_{\mathcal{D}_{t-1}})\big)}=
\nonumber
\\
\sum_{{N}_{T}|n_{t}=n}\frac{p(\bm{L}_{\mathcal{D}_{t-1}})\cdot p({N}_{\mathcal{D}_{t-1}}|\bm{L}_{\mathcal{D}_{t-1}})\cdot p({N}_{T}|{N}_{\mathcal{D}_{t-1}})}{P\big(E(\bm{L}_{\mathcal{D}_{t-1}})\big)}=
\nonumber
\\
\frac{p(\bm{L}_{\mathcal{D}_{t-1}})}{P\big(E(\bm{L}_{\mathcal{D}_{t-1}})\big)}\sum_{{N}_{T}|n_{t}=n}p({N}_{T}|\bm{L}_{\mathcal{D}_{t-1}})=
\nonumber
\\
\frac{p(\bm{L}_{\mathcal{D}_{t-1}})}{P\big(E(\bm{L}_{\mathcal{D}_{t-1}})\big)}p(n_{t}=n|\bm{L}_{\mathcal{D}_{t-1}})=\frac{p(\bm{L}_{\mathcal{D}_{t-1}})}{P\big(E(\bm{L}_{\mathcal{D}_{t-1}})\big)}w_{t}^{n}
\nonumber
\end{eqnarray}
Let us show that the value of $n$-independent normalizing constant $\frac{p(\bm{L}_{\mathcal{D}_{t-1}})}{P\big(E(\bm{L}_{\mathcal{D}_{t-1}})\big)}$ is equal to $1$. Indeed, 

\begin{eqnarray}1=\sum_{s\in\mathcal{S}}w_{t}^{s}=
\sum_{n\in\mathcal{N}}\widehat{w}_{t}^{n}=\frac{p(\bm{L}_{\mathcal{D}_{t-1}})}{P\big(E(\bm{L}_{\mathcal{D}_{t-1}})\big)}\sum_{n\in\mathcal{N}}w_{t}^{n}=\frac{p(\bm{L}_{\mathcal{D}_{t-1}})}{P\big(E(\bm{L}_{\mathcal{D}_{t-1}})\big)}.
\nonumber
\end{eqnarray}

We conclude that $\widehat{w}_{t}^{n}=w_{t}^{n}$ for all $t=1,\dots T$ and $n\in\mathcal{N}$. Thus, we proved that algorithms $\mathcal{G}_{\text{base}}(P)$ and $\mathcal{G}(p)$ have exactly the same losses. Let $H_{T}$ be the cumulative loss of these algorithms. Then, by using part 1 of the proof of the theorem we conclude:

\begin{eqnarray}
H_{T}\leq 
-\frac{1}{\eta}\ln \bigg[\mathbb{E}_{P(s)}\big[e^{-\eta L_{T}^{s}}\big]\bigg]+\eta\frac{H^{2}}{8}T+\eta\frac{H^{2}}{4}\sum_{t=1}^{T}D_{t}=
\nonumber
\\
-\frac{1}{\eta}\ln \bigg[\mathbb{E}_{p({N}_{T})}\big[e^{-\eta L_{T}^{{N}_{T}}}\big]\bigg]+\eta\frac{H^{2}}{8}T+\eta\frac{H^{2}}{4}\sum_{t=1}^{T}D_{t}
\nonumber
\end{eqnarray}
and finish the proof.\end{proof}

\section{Experiments}
\label{sec-experiments}

We empirically compare developed non-replicating algorithm \ref{algorithm-delayed-hedge} ($\mathcal{G}_{\text{base}}$) and algorithm \ref{algorithm-delayed-fixed-share} ($\mathcal{G}_{\text{fs}}$) with their analogous replicated ones obtained from non-delayed Hedge and Fixed Share by using meta-algorithm BOLD \cite{joulani2013online}. 

To begin with, we recall the main idea of replicating meta-algorithm BOLD. For the sequence of the delays $\{D_{t}\}_{t=1}^{T}$ meta-algorithm BOLD splits the time line into disjoint subsequences. Each subsequence $\{t_{1}<\dots<t_{S}\}$ satisfies ${t_{s}+D_{t_{s}}<t_{s+1}}$, so it is possible to run an independent copy of some non-delayed algorithm $\mathcal{A}$ on the subsequence. For simplicity we assume that all the delays $D_{t}$ are known to the BOLD beforehand. Thus, the meta-algorithm can choose the optimal learning rate for each copy of $\mathcal{A}$ depending on the length of the corresponding subsequence. For more details about algorithm \textbf{BOLD} please refer to the original paper \cite{joulani2013online}.

We use BOLD$(\mathcal{H}_{\text{base}})$ and BOLD$(\mathcal{H}_{\text{fs}})$ to denote replicated Hedge and Fixed Share respectively.

We conduct the experiments on the \textbf{artificial data}. The artificial data is widely used to illustrate  the performance of the Hedge-like algorithms (see \cite{HeW98,BoW2002,erven2011adaptive,de2014follow}).

To generate the data we use schemes similar to the ones from \cite{erven2011adaptive}. In all our experiments we set $N=4$ experts and use binary losses, i.e. $\{0, 1\}$. Thus, we set $H=1$. The length of the game is $T=10 000$.

We sample $l_{t}^{n}\sim \text{Bernoulli}(q^{n})$, i.i.d. random variables for all ${n=1,2,3,4}$ and ${t=1,2,\dots T}$. We use two variants of $\bm{q}$: the first one is $${\bm{q}_{1}=[q^{1}, q^{2}, q^{3}, q^{4}]=[0.35, 0.4, 0.45, 0.5]},$$
when all the experts suffer approximately similar losses; the second one,
$$\bm{q}_{2}=[q^{1}, q^{2}, q^{3}, q^{4}]=[0.2, 0.4, 0.5, 0.7],$$
when experts differ a lot.

The sequence of delays is random. Each $D_{t}$ is sampled from Poisson distribution with known to the learner mean $\lambda$, i.e. $D_{t}\sim\text{Poisson}(\lambda)$.

Note that all the computational results are \textbf{averaged} on ${R=250}$ random realizations of data (losses, delays) for all considered parameters ($\bm{q}, \lambda$).

\subsection{Experiments with Hedge}
\label{sec-exp-hedge}
In this subsection we compare non-replicating algorithm $\mathcal{G}_{\text{base}}$ and replicating algorithm BOLD$(\mathcal{H}_{\text{base}})$.

For each copy of $\mathcal{H}_{\text{base}}$ started by BOLD on the subsequence of length $S$ we use its optimal learning rate
\begin{equation}\eta^{*}=\argmin_{\eta>0}\big[\frac{\ln N}{\eta}+\eta\frac{H^{2}S}{8}\big]=\frac{2}{H}\sqrt{\frac{2\ln N}{S}}.
\label{learning-rate-subprocess}
\end{equation}

Note that BOLD($\mathcal{H}_{\text{base}}$) runs roughly $\approx[1+\mathbb{E}D]$ copies of $\mathcal{H}_{\text{base}}$, each of length $\approx \frac{T}{[\mathbb{E}D+1]}$ with learning rate\footnote{In the case $D_{t}\equiv D=\mathbb{E}D$ for all $t$, all approximations become equalities.} $$\eta^{\mathcal{H}_{\text{base}}}\approx \frac{2}{H}\sqrt{\frac{2\ln N}{T}(1+\mathbb{E}D)}.$$ 
Thus, in order to equalize the learning speed of $\mathcal{G}_{\text{base}}$ and BOLD($\mathcal{H}_{\text{base}}$), it is fair to assign $[1+\mathbb{E}D]$ times lower learning rate 
\begin{equation}
\eta^{\mathcal{G}_{\text{base}}}=[1+\mathbb{E}D]^{-1}\cdot \eta^{\mathcal{H}_{\text{base}}}=\frac{2}{H}\sqrt{\frac{2\ln N}{T(1+\mathbb{E}D)}}
\label{learning-rate-main}
\end{equation}
to algorithm $\mathcal{G}_{\text{base}}$. The usage of such $\eta$ leads to $O(\sqrt{T(1+\mathbb{E}D)})$ regret bound (see Subsection \ref{hedge-delayed}).

For integer values of $\lambda=\mathbb{E}D\in [0, 250]$, we compare the total regret $R_{T}$ of $\mathcal{G}_{\text{base}}$ and BOLD($\mathcal{H}_{\text{base}}$) with respect to the best expert. The resulting empirical dependence is shown in Figures \ref{figure:hedge-similar}, \ref{figure:hedge-distant} for losses generated with the use of $\bm{q}_{1}$ and $\bm{q}_{2}$ respectively.

% \begin{figure}[!htb]
% \centering
% \subfloat[\textbf{Hedge:} similar experts.]{
%   \includegraphics[clip,scale=0.2]{hedge-similar.eps}
%   \label{figure:hedge-similar}
% }

% \subfloat[\textbf{Hedge:} diverse experts.]{
%   \includegraphics[clip,scale=0.2]{hedge-distant.eps}
%   \label{figure:hedge-distant}
% }
% \label{figure:hedge-results}
% \caption{Total regret $R_{T}$ w.r.t. the best expert as a function of the Expected Delay $\lambda=\mathbb{E}D$ for non-replicated $\mathcal{G}_{\text{base}}$ and replicated BOLD$(\mathcal{H}_{\text{base}})$.}
% \end{figure}

\begin{figure}[!htb]
\centering     %%% not \center
\subfigure[\textbf{Hedge:} similar experts.]{\label{figure:hedge-similar}\includegraphics[width=0.48\textwidth]{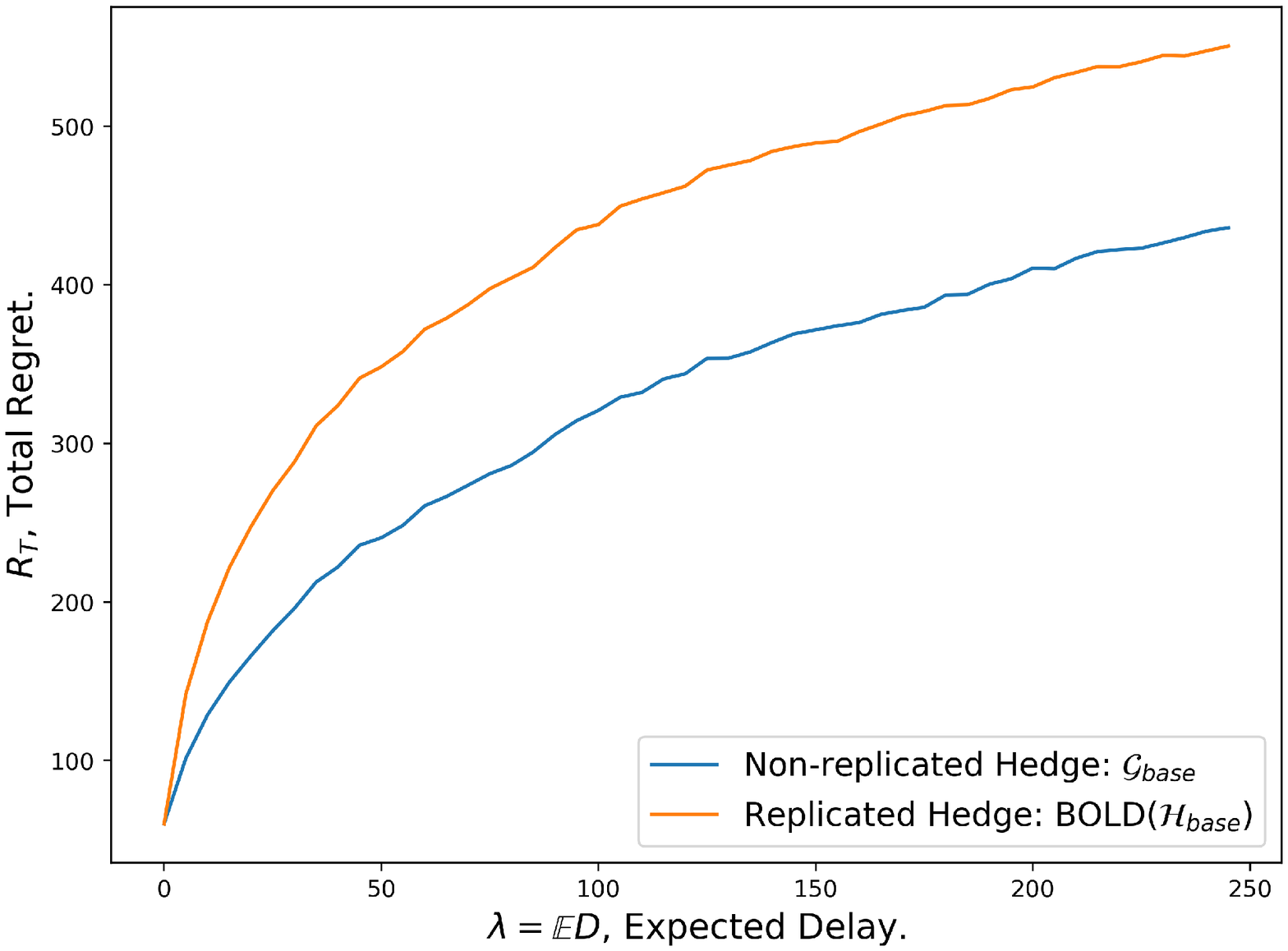}}
\subfigure[\textbf{Hedge:} diverse experts.]{\label{figure:hedge-distant}\includegraphics[width=0.48\textwidth]{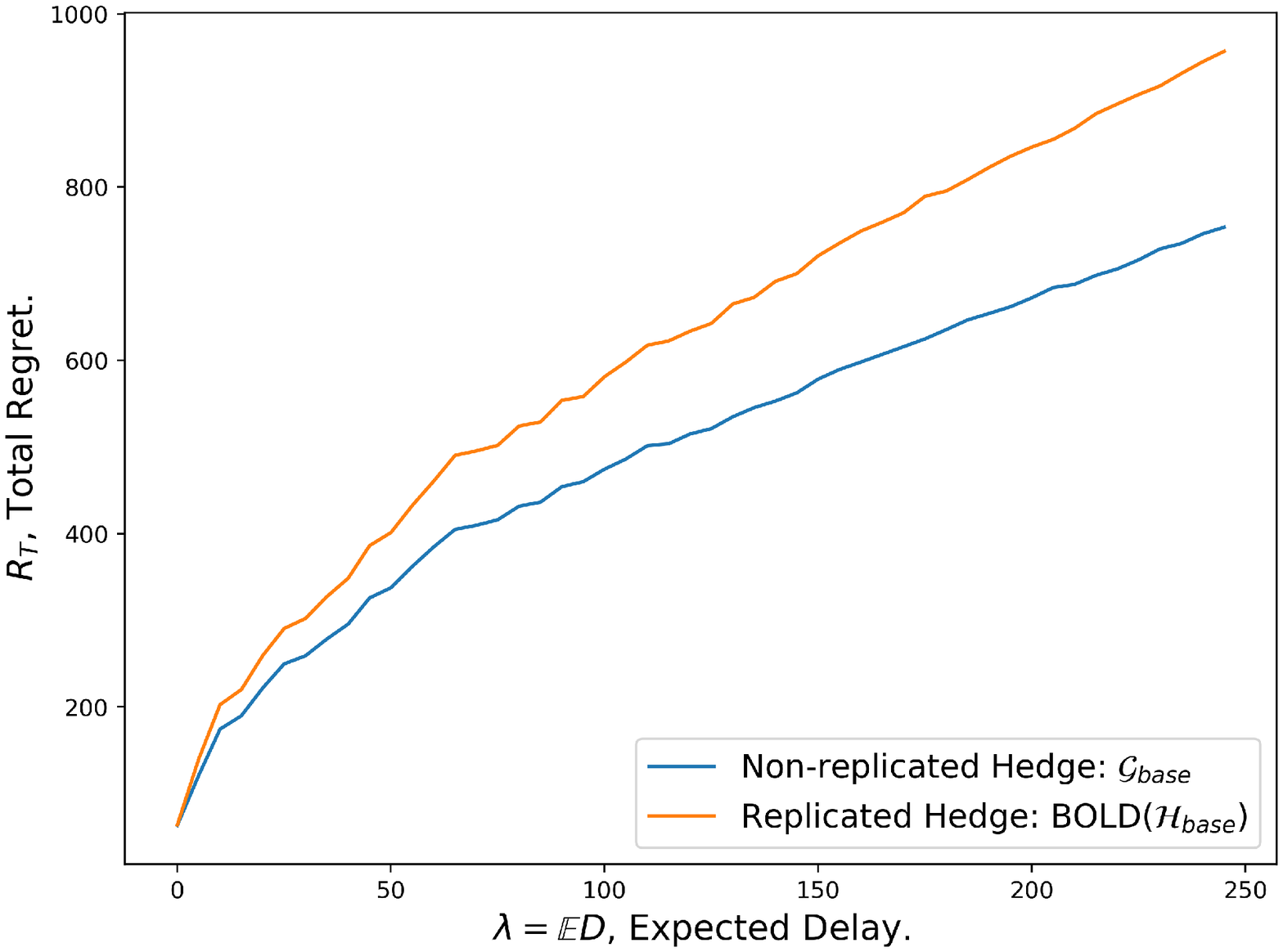}}
\caption{Total regret $R_{T}$ w.r.t. the best expert as a function of the Expected Delay $\lambda=\mathbb{E}D$ for non-replicated $\mathcal{G}_{\text{base}}$ and replicated BOLD$(\mathcal{H}_{\text{base}})$.}
\end{figure}

We discuss the results in Section \ref{sec-experiments-discussion} below. 

\subsection{Experiments with Fixed Share}

In this subsection we compare non-replicating algorithm $\mathcal{G}_{\text{fs}}$ and replicating algorithm BOLD$(\mathcal{H}_{\text{fs}})$.

We set $K=10$ switches and generate datasets which have $K$ switches of the best expert. To create such a dataset, we randomly select $K$ time steps ${t_{1}<t_{2}<\dots<T_{K}}$. On each $k$-th segment $[t_{k}+1,t_{k+1}]$ (for $k=0,1,\dots, K$ and $t_{0}=0, t_{K+1}=T$) we fix random permutation $\sigma_{k}$ on the set of $N$ elements and sample the losses of expert $n=1,2,3,4$ from Bernoulli$(q^{\sigma_{k}(n)})$. Thus, we obtain the sequence of losses which has up to $K$ switches of the best expert. We also assume that the learner does not know $K$ in advance.

% The logic of this Subsection is similar to Subsection \ref{sec-exp-hedge}. We compare two versions of non-replicated algorithm $\mathcal{G}_{\text{fs}}$ (optimal, accelerated) and a replicated version BOLD$(\mathcal{H}_{\text{fs}})$.

% The optimal learning rate for $\mathcal{G}_{\text{fs}}$ is defined as the one that minimizes\footnote{Since the learner does not know $K$ in advance, we simply substitute $K=0$.} the regret bound \eqref{fixed-share-regret-bound} for Fixed Share:
% $$\eta^{\text{opt}}=\frac{2}{H}\sqrt{\frac{2(\ln N+\ln T)}{T(1+2\mathbb{E}D)}}.$$

% The learning rate for $\mathcal{H}_{\text{fs}}$ subprocess of length $S$ and the accelerated learning rate for $\mathcal{G}_{\text{fs}}$ are defined by
% \begin{equation}
% \eta^{\mathcal{H}_{\text{base}}}=\frac{2}{H}\sqrt{\frac{2(\ln N+\ln S)}{S}(1+\mathbb{E}D)}
% \qquad \text{and} \qquad \eta^{\text{acc}}=\frac{2}{H}\sqrt{\frac{2(\ln N+\ln T)}{T(1+\mathbb{E}D)}}
% \nonumber
% \end{equation}
% respectively.

% In order not to overburden the reader, we use the same learning rates as in the previous subsection. For every copy of $\mathcal{H}_{\text{fs}}$ generated by BOLD, the learning rate is defined by \eqref{learning-rate-subprocess}. For the $\mathcal{G}_{\text{fs}}$ the learning rates (optimal, accelerated) are given by equations \eqref{learning-rate-optimal}, \eqref{learning-rate-accelerated}.

In order not to overburden the reader, we use the same learning rates as in the previous subsection. For every copy of $\mathcal{H}_{\text{fs}}$ generated by BOLD, the learning rate is defined by \eqref{learning-rate-subprocess}. For the $\mathcal{G}_{\text{fs}}$ the learning rate is given by \eqref{learning-rate-main}.

% For integer values of $\lambda=\mathbb{E}D\in [0, 250]$ we compare the total regret $R_{T}$ w.r.t. the best shifting sequence of experts with $K$ switches (at time points $t_{1},\dots,t_{k}$) of $\mathcal{G}_{\text{fs}}$ and BOLD($\mathcal{H}_{\text{fs}})$. The resulting empirical dependence is shown in Figures \ref{figure:fixed-share-similar}, \ref{figure:fixed-share-distant} for losses generated with the use of $\bm{q}_{1}$ and $\bm{q}_{2}$ respectively.

% For integer values of $\lambda=\mathbb{E}D\in [0, 250]$ we compare the total regret $R_{T}$ w.r.t. the best shifting sequence of experts with $K$ switches (at time points $t_{1},\dots,t_{k}$) of $\mathcal{G}_{\text{fs}}$ (optimal/accelerated learning rate) and BOLD($\mathcal{H}_{\text{fs}})$. The resulting empirical dependence is shown in Figures \ref{figure:fixed-share-similar}, \ref{figure:fixed-share-distant} for losses generated with the use of $\bm{q}_{1}$ and $\bm{q}_{2}$ respectively.

% \begin{figure}[!htb]

% \subfloat[\textbf{Fixed Share:} similar experts.]{
%   \includegraphics[clip,scale=0.4]{fixed-share-similar.eps}
%   \label{figure:fixed-share-similar}
% }

% \subfloat[\textbf{Fixed Share:} diverse experts.]{
%   \includegraphics[clip,scale=0.4]{fixed-share-distant.eps}
%   \label{figure:fixed-share-distant}
% }
% \label{figure:fixed-share-results}
% \caption{Total regret $R_{T}$ w.r.t. the best sequence of experts with no more than $K=10$ shifts as a function of the Expected Delay $\lambda=\mathbb{E}D$ for non-replicated $\mathcal{G}_{\text{fs}}$ and replicated BOLD$(\mathcal{H}_{\text{fs}})$.}
% \end{figure}

\begin{figure}[!htb]
\centering     %%% not \center
\subfigure[\textbf{Fixed Share:} similar experts.]{\label{figure:fixed-share-similar}\includegraphics[width=0.48\textwidth]{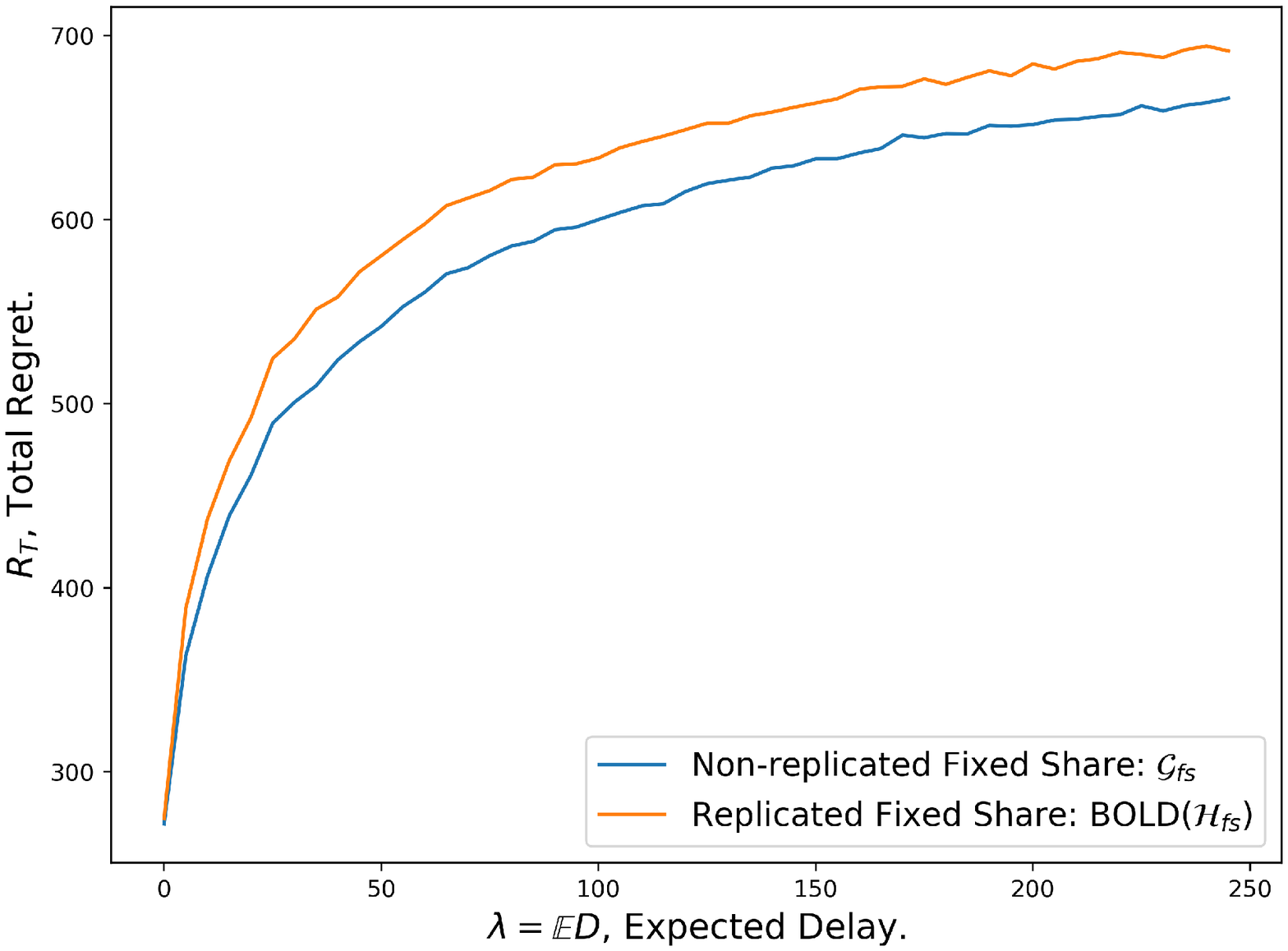}}
\subfigure[\textbf{Fixed Share:} diverse experts.]{\label{figure:fixed-share-distant}\includegraphics[width=0.48\textwidth]{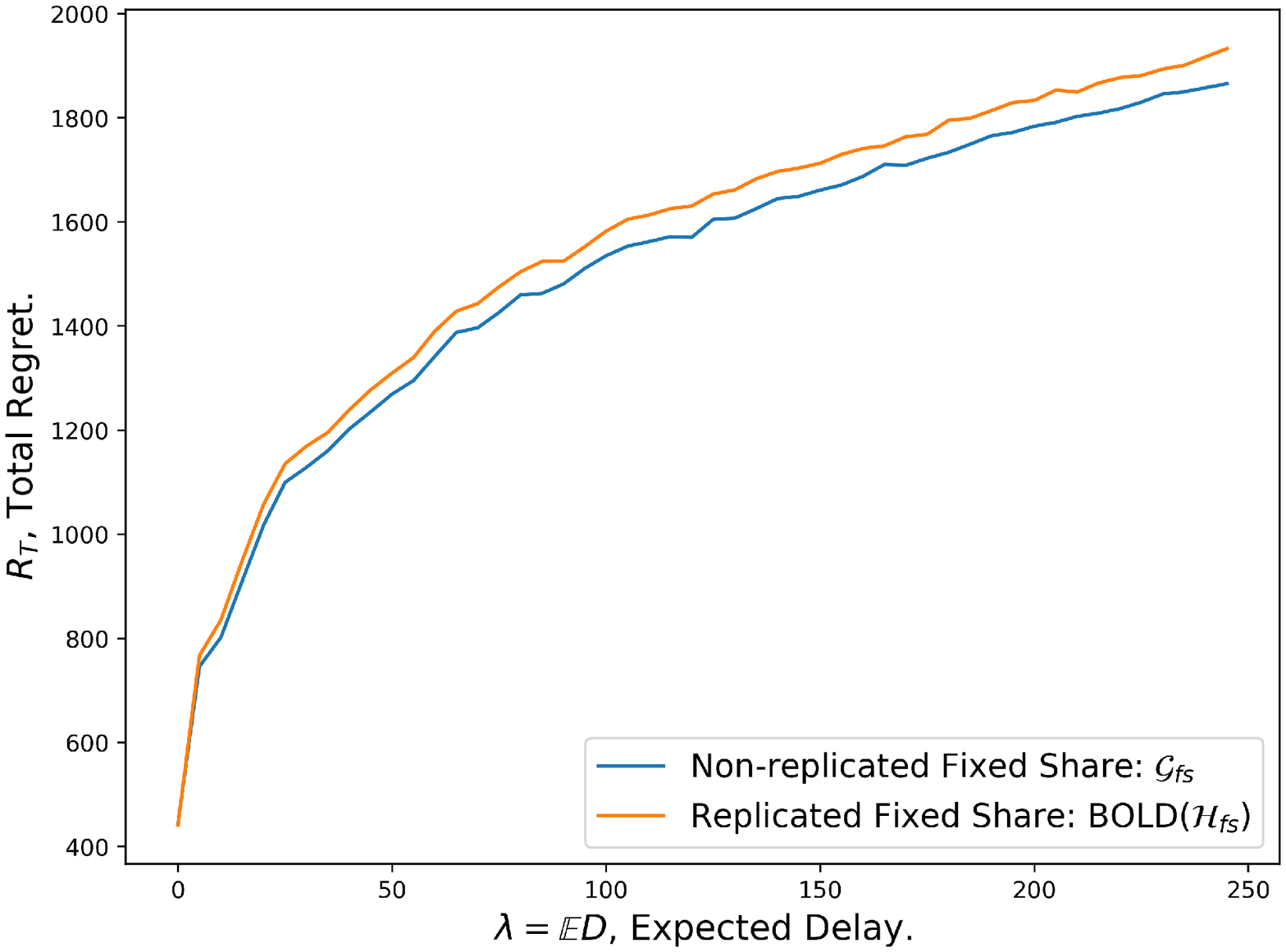}}
\caption{Total regret $R_{T}$ w.r.t. the best sequence of experts with no more than $K=10$ shifts as a function of the Expected Delay $\lambda=\mathbb{E}D$ for non-replicated $\mathcal{G}_{\text{fs}}$ and replicated BOLD$(\mathcal{H}_{\text{fs}})$.}
\end{figure}
We discuss the results in Section \ref{sec-experiments-discussion} below. 

\subsection{Discussion}
\label{sec-experiments-discussion}

In all Figures \ref{figure:hedge-similar}, \ref{figure:hedge-distant}, \ref{figure:fixed-share-similar}, \ref{figure:fixed-share-distant} we see that the non-replicated algorithms outperform their corresponding replicating opponents.

For Hedge algorithm from Figures \ref{figure:hedge-similar}, \ref{figure:hedge-distant} we also conclude that with the increase of the expected delay $\mathbb{E}D$ the gap between performance of non-replicating Hedge ($\mathcal{G}_{\text{base}}$) and replicating BOLD$(\mathcal{H}_{base})$ increases. Indeed, the bigger the expected delay is, the more infrequent the separate learning processes generated by BOLD become and the less data they see. Nevertheless, while each base copy of $\mathcal{H}_{\text{base}}$ runs on $\approx \frac{1}{1+\mathbb{E}D}$ times less data than the non-replicated $\mathcal{G}_{\text{base}}$, it uses $\approx(1+\mathbb{E}D)$ times higher learning rate, which should balance the learning speed with the non-replicated $\mathcal{G}_{\text{base}}$.

Note that Hedge is equal to Online Mirror Descent (OMD) with Entropic Regularization (see e.g. \cite{shalev2012online}). OMD runs Online Gradient Descent (OGD)
$$\bm{x}_{t}\leftarrow \bm{x}_{t-1}-\eta\cdot \bm{l}_{t}$$
in the mirrored space $\mathbb{R}^{N}$ and after each gradient step transforms the mirrored weight $\bm{x}_{t}$ into primal weight 
$$\bm{w}_{t+1}=\text{SoftMax}(\bm{x}_{t})=\bigg[\frac{e^{x^{1}_{t}}}{\sum_{n=1}^{N}e^{x^{n}_{t}}},\dots,\frac{e^{x^{N}_{t}}}{\sum_{n=1}^{N}e^{x^{n}_{t}}}\bigg],$$
so that $\bm{w}_{t+1}\in \Delta(\mathcal{N})$ is the decision of the algorithm on weight allocation.

In the case of i.i.d. experts losses, the mirrored estimates of $\bm{x}_{t}$ for both $\mathcal{G}_{\text{base}}$ on $t$ observations and $\mathcal{H}_{\text{base}}$ on $\approx\frac{1}{1+\mathbb{E}D}t$ observations have the same expectation. Indeed,
\begin{eqnarray}\mathbb{E}(\bm{x}_{t}^{\mathcal{G}_{\text{base}}})=\mathbb{E}\sum_{\tau=1}^{t}\big[-\eta^{\mathcal{G}_{\text{base}}}\cdot \bm{l}_{\tau}\big]=-\eta^{\mathcal{G}_{\text{base}}}\cdot\sum_{\tau=1}^{t}\mathbb{E} \bm{l}_{\tau}=-\eta^{\mathcal{G}_{\text{base}}}\cdot\big( t\cdot\mathbb{E}\bm{l}\big)=
\label{mirrored-expectation-base}
\\
 -\big(\frac{\eta^{\mathcal{H}_{\text{base}}}}{1+\mathbb{E}D})\cdot \big( t\cdot\mathbb{E}\bm{l}\big)=-\big(\eta^{\mathcal{H}_{\text{base}}})\cdot \underbrace{\big( \frac{t}{1+\mathbb{E}D}\cdot\mathbb{E}\bm{l}\big)}_{\approx \sum\limits_{\tau\in \text{SP}(t)}\mathbb{E}\bm{l}_{\tau}}\approx
\label{mirrored-expectation-intermediate}
 \\
 \sum\limits_{\tau\in \text{SP}(t)}\big[-\eta^{\mathcal{H}_{\text{base}}}\cdot\mathbb{E}\bm{l}_{\tau}\big]=\mathbb{E}(\bm{x}_{t}^{\text{BOLD}}),
\label{mirrored-expectation-bold}
\end{eqnarray}
where we use $\text{SP}(t)$ to denote the set of all time steps $\tau\leq t$ included in the \textbf{s}eparate learning \textbf{p}rocess (generated by BOLD) that is used at the step $t$. In the transition between lines \eqref{mirrored-expectation-base} and \eqref{mirrored-expectation-intermediate} we use definition \eqref{learning-rate-main} of the learning rates. In line \eqref{mirrored-expectation-bold} we note that the size of the set $\text{SP}(t)$ is $\approx \frac{t}{1+\mathbb{E}D}$.

Same as in \eqref{mirrored-expectation-base}-\eqref{mirrored-expectation-bold}, we compare the co-variance matrices of the estimates of the mirrored estimates of $\bm{x}_{t}$ obtained by $\mathcal{G}_{\text{base}}$ and BOLD. Again, using the i.i.d. assumption we derive

\begin{eqnarray}\mathbb{V}(\bm{x}_{t}^{\mathcal{G}_{\text{base}}})=
\mathbb{V}\sum_{\tau=1}^{t}\big[-\eta^{\mathcal{G}_{\text{base}}}\cdot \bm{l}_{\tau}\big]=
\big(\eta^{\mathcal{G}_{\text{base}}}\big)^{2}\cdot\sum_{\tau=1}^{t}\mathbb{V} \bm{l}_{\tau}=
\nonumber
\\
\big(\eta^{\mathcal{G}_{\text{base}}}\big)^{2}\cdot\big( t\cdot\mathbb{V}\bm{l}\big)=
 \big(\frac{\eta^{\mathcal{H}_{\text{base}}}}{1+\mathbb{E}D}\big)^{2}\cdot \big( t\cdot\mathbb{V}\bm{l}\big)=
 \nonumber
 \\
 \frac{\big(\eta^{\mathcal{H}_{\text{base}}})^{2}}{1+\mathbb{E}D}\cdot \underbrace{\big( \frac{t}{1+\mathbb{E}D}\cdot\mathbb{V}\bm{l}\big)}_{\approx \sum\limits_{\tau\in \text{SP}(t)}\mathbb{V}\bm{l}_{\tau}}\approx 
 \frac{1}{1+\mathbb{E}D}\cdot \sum\limits_{\tau\in \text{SP}(t)}\big[(\eta^{\mathcal{H}_{\text{base}}})^{2}\cdot\mathbb{V}\bm{l}_{\tau}\big]=
 \nonumber
 \\
 (1+\mathbb{E}D)^{-1}\cdot \mathbb{V}(\bm{x}_{t}^{\text{BOLD}}).
\nonumber
\end{eqnarray}
Note that all the described co-variance matrices are diagonal because we consider the case when the losses of different experts are independent.

We see that while the expectation of the estimates of the mirrored weight $\bm{x}_{t}$ is equal for both non-replicated $\mathcal{G}_{\text{base}}$ and replicated BOLD$(\mathcal{H}_{\text{base}})$, the variance differs $1+\mathbb{E}D$ times. In particular, this means that the distribution of mirrored weights $\bm{x}_{t}$ for these two algorithms differs. The mirrored weight of $\mathcal{G}_{\text{base}}$ is more robust than the corresponding weight of a copy of $\mathcal{H}_{\text{base}}$. As we see from the experiments, these robustness of mirrored weight $\bm{x}_{t}$ also leads to robustness of the primal weights $\bm{w}_{t+1}$ and results in better performance.

If the data does not behave like stochastic, e.g. is maximally adversarial, the above argument obviously does not work, and the replicated algorithms may outperform their non-replicated analogues.

We also note another important advantage of the non-replicating algorithms. They are more \textbf{interpretable} than their replicated analogues. The weights obtained by non-replicated algorithms are smooth (thus, more interpretable), whereas the weights of replicated algorithms are smooth only inside every domain of the independent learning subprocess.

To illustrate this, we plot the weight evolution of experts obtained by $\mathcal{G}_{\text{fs}}$ and BOLD$(\mathcal{H}_{\text{fs}})$ in a single experiment with $\mathbb{E}D=40$ and $K=10$ experts' switches with experts' losses generated using $\bm{q}_{2}$. The weight evolution on time interval $(4200, 4300)$ is shown in Figures \ref{figure:evolution-non-replicated} and \ref{figure:evolution-replicated}. One may clearly see that the experts' weights of replicated algorithm in Figure \ref{figure:evolution-replicated} look like uninterpretable noise (because the weights of separate learning processes significantly differ). 

\begin{figure}[!htb]
\centering     %%% not \center
\subfigure[\textbf{Non-replicated} Fixed Share.]{\label{figure:evolution-non-replicated}\includegraphics[width=0.48\textwidth]{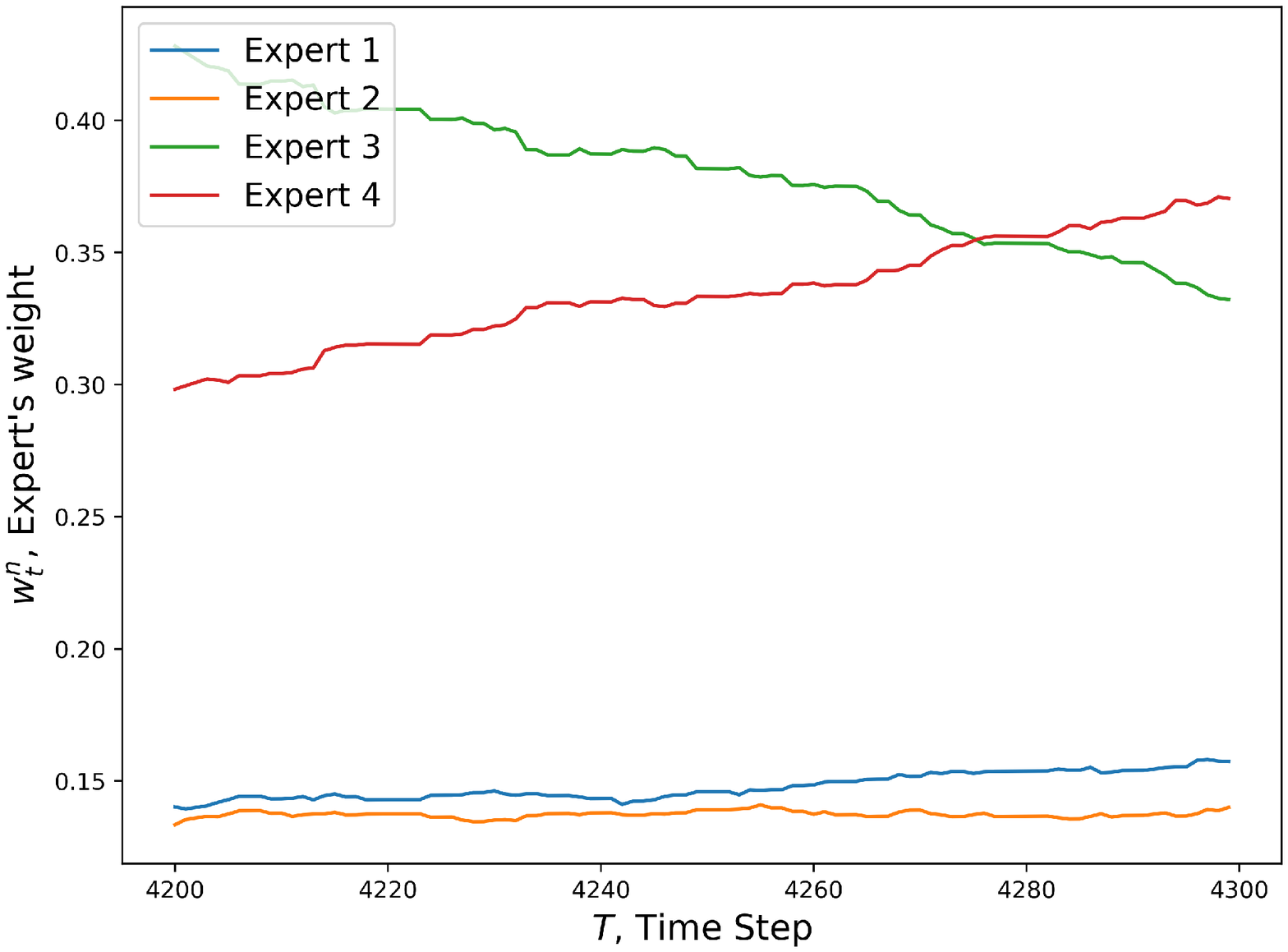}}
\subfigure[\textbf{Replicated} Fixed Share.]{\label{figure:evolution-replicated}\includegraphics[width=0.48\textwidth]{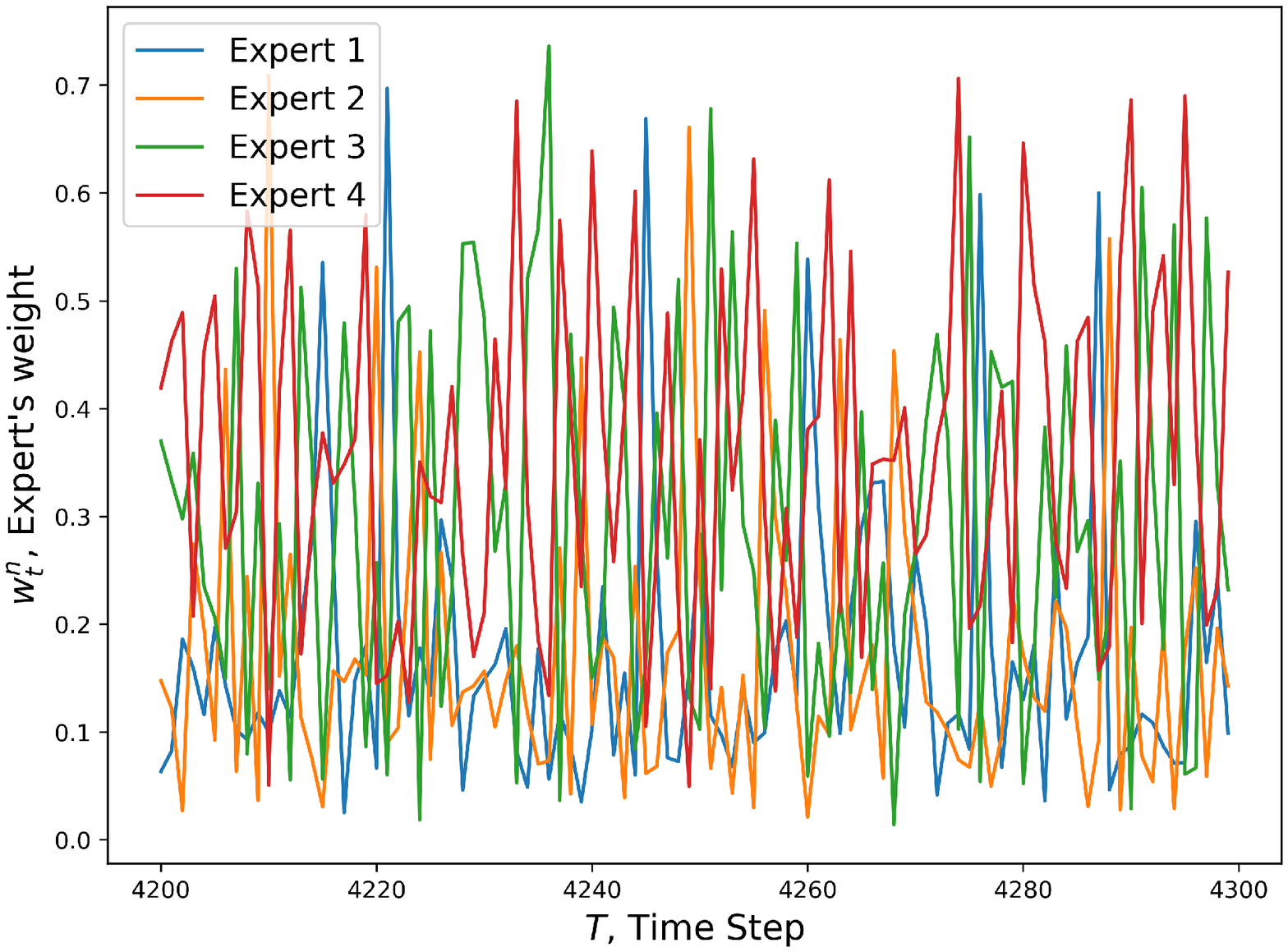}}
\caption{Evolution of weights of non-replicated $\mathcal{G}_{\text{fs}}$ and replicated BOLD$(\mathcal{H}_{\text{fs}})$ on the same data during time steps (4200, 4300).}
\end{figure}

We also attach the plot of the full weight evolution of the non-replicated algorithm in Figure \ref{figure:evolution-non-replicated-all}.

\begin{figure}
    \centering
    \includegraphics[width=0.98\textwidth]{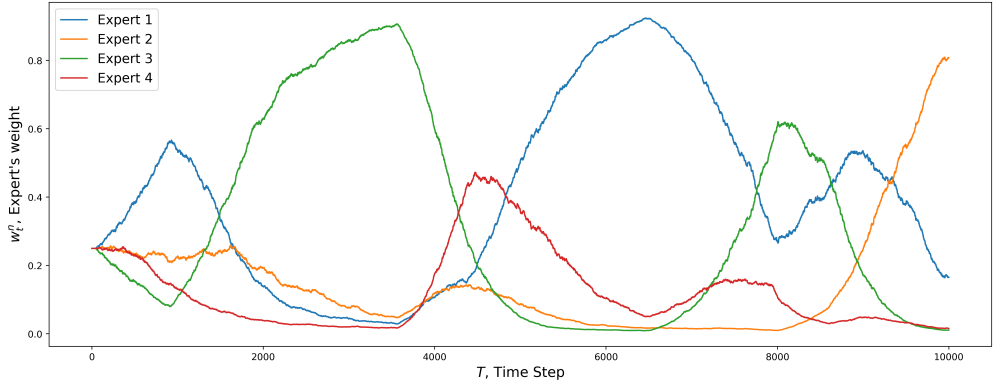}
    \caption{Evolution of weights of \textbf{non-replicated} algorithm during the full game.}
    \label{figure:evolution-non-replicated-all}
\end{figure}

To conclude, it seems that the replicated algorithms outperform non-replicated ones on the stochastic-like data. It would be interesting to obtain some concrete \textbf{empirical condition} on adversarial data under which the non-replicated algorithms perform better than their replicated analogues. This problem serves as the challenge for our further research.

% \begin{figure}[!htb]

% \subfloat[Evolution of weights of \textbf{non-replicated} algorithm during the full game.]{
%   \includegraphics[clip,scale=0.4]{evolution-non-replicated-all}
%   \label{figure:evolution-non-replicated-all}
% }

% \subfloat[Evolution of weights of \textbf{non-replicated} algorithm during time steps (1700, 1800).]{
%   \includegraphics[clip,scale=0.4]{evolution-non-replicated}
%   \label{figure:evolution-non-replicated}
% }

% \subfloat[Evolution of weights of \textbf{replicated} algorithm during time steps (1700, 1800).]{
%   \includegraphics[clip,scale=0.4]{evolution-replicated.eps}
%   \label{figure:evolution-replicated}
% }
% \label{figure:evolution}
% \caption{Evolution of weights of non-replicated $\mathcal{G}_{\text{fs}}$ and replicated BOLD$(\mathcal{H}_{\text{fs}})$ on the same data.}
% \end{figure}

\section{Conclusion}
In the article we developed the general hedging algorithm $\mathcal{G}$ (based on classical Hedge) for the delayed feedback experts' weight allocation (see Section \ref{sec-algorithm}, Algorithm \ref{algorithm-main}). The developed algorithm is applicable both to hedging countable and continuous sets of experts. Thanks to our main result (Theorem \ref{theorem-main-loss-bound-D}), we can bound its loss or regret with respect to the switching sequence of experts.

We described two examples of applications of algorithm $\mathcal{G}$ for delayed feedback setting. Algorithm \ref{algorithm-delayed-hedge} ($\mathcal{G}_{\text{base}}$, Subsection \ref{hedge-delayed}) is an extension of the classical Hedge for the delayed feedback. Algorithm \ref{algorithm-delayed-fixed-share} ($\mathcal{G}_{\text{fs}}$, Subsection \ref{sec-fs}) is the adaptation of classical Fixed Share. Both algorithms are non-replicated, which means that they use all the observed data to make the decision (in contrast to existing meta-approaches to delayed feedback setting).

It seems that the general probabilistic model which we described can be enhanced even more. First of all, it is reasonable to consider dynamic time-dependent learning rates $\eta_{t}$ for different time steps $t$.\footnote{The usual choice of dynamic learning rate in the non-delayed setting is $\eta_{t}\propto \frac{1}{\sqrt{t}}$.} This may rid the learner from choosing the learning rate beforehand. Secondly, it is possible to consider different observation probabilities \eqref{observation-probability} (or potential, see \cite{cesa-bianchi}). The different choice may allow to obtain the generalized versions and loss bound of Theorem \ref{theorem-main-loss-bound-D} for many other algorithms based on multiplicative weights (e.g. MW2 \cite{cesa2007improved}). The described statements serve as the challenge for our further research.

\section*{Acknowledgements}
\noindent The research was partially supported by the Russian Foundation for Basic Research grant 16-29-09649 ofi m.

\bibliographystyle{elsarticle-num-names}

%\bibliography{references}

\clearpage
\appendix

\section{Math Tools}
\label{sec-appendix-math}

In this appendix we describe the math tools that we use in out article. We start with the well-known Hoeffding's \& Pinsker's inequalities and then state and prove the important Lemmas (used in the proof of our main Theorem \ref{theorem-main-loss-bound-D}).

\vspace{1mm}

\noindent\textbf{Hoeffding's inequality.} {\emph Let $X\in[a,b]\subset \mathbb{R}$ be a random variable. Then,
\begin{equation}\ln \mathbb{E}e^{sX}\leq s\mathbb{E}X+s^{2}\frac{(b-a)^2}{8}
\label{hoeffding}
\end{equation}
for all $s\in\mathbb{R}$.}

\vspace{2mm}

\noindent\textbf{Pinsker's inequality.} Let $p(x)$ and $q(x)$ be probabilities (or densities) of $x\in X$ for two discrete (continuous) distributions over discrete (continuous) set $X\subset \mathbb{R}^{N}$. Then
\begin{equation}
\max_{X'\subset X}|p(X')-q(X')|\leq \sqrt{\frac{1}{2}KL(p||q)},
\label{pinsker}
\end{equation}
where $KL(p||q)$ is Kullback–Leibler divergence between $p$ and $q$.

\vspace{2mm}

The following technical Lemma plays an important role in the proof of Theorem \ref{theorem-main-loss-bound-D} (Section \ref{sec-proofs}).
\begin{lemma}
\label{lemma-change-bound}
Let $X\subset \mathbb{R}^{N}$ be a countable (or continuous) set. Let $p(x)$ and $q(x)$ denote probabilities (or densities) of two random variables with values in $X$. Let $a:X\rightarrow \mathbb{R}$ be a measurable function such that for all $x\in X$ we have $-\frac{1}{\eta}\ln a(x)\in [0, C]$. Then if $q(x)\propto p(x)\cdot a(x)$, the following holds true\footnote{In the continuous case the sum should be replaced by the integral.}
\begin{equation}
\sum_{x: p(x)\geq q(x)}[p(x)-q(x)]\leq \frac{\eta C}{4}.
\label{change-bound}
\end{equation}
\end{lemma}
\begin{proof}
Apply Pinsker's inequality \ref{pinsker} for $p(\cdot)$ and $q(x)$ and obtain
\begin{equation}
\sum_{x: p(x)\geq q(x)}[p(x)-q(x)]\leq \sqrt{\frac{1}{2}KL(p||q)}.
\label{pinsker-to-lemma}
\end{equation}
Note that $q(x)=\frac{p(x)a(x)}{\sum_{x'\in X}p(x')a(x')}$. We compute the divergence
\begin{eqnarray}
\nonumber
KL(p||q)=\sum_{x\in X}p(x)\ln \frac{p(x)}{q(x)}=
\sum_{x\in X}p(x)\ln \frac{\sum_{x'\in X}p(x')a(x')}{a(x)}=
\\
\ln \sum_{x\in X}p(x)a(x) -\sum_{x\in X}p(x)\ln a(x)=
\nonumber
\\
\eta \bigg[\sum_{x\in X}p(x)l^{x}-\frac{1}{\eta}\ln \sum_{x\in X}p(x)e^{-\eta l^{x}}\bigg]\leq
\eta \cdot \frac{\eta C^{2}}{8}=\frac{\eta^{2}C^{2}}{8},
\label{mix-gap-in-pinsker}
\end{eqnarray}
where in \eqref{mix-gap-in-pinsker} we denote $l^{x}=-\frac{1}{\eta}\ln a(x)\in [0, C]$ (for $x\in X$) and use Hoeffding's inequality \eqref{hoeffding} for variable which is equal to $l^{x}$ w.p. $p(x)$. To finish, we obtain the bound \eqref{change-bound} by combining \eqref{pinsker-to-lemma} with the upper bound \eqref{mix-gap-in-pinsker}.
\end{proof}

\begin{lemma} Let $T>0$ be an integer and $\{D_{t}\}_{t=1}^{T}$ be the sequence of integer delays such that $t+D_{t}\leq T$. Let $\mathcal{D}_{t}=\{\tau|\tau+D_{\tau}\leq t\}$. Then
\begin{equation}
\sum_{t=1}^{T-1}|\mathcal{D}_{t}|+\sum_{t=1}^{T-1}D_{t}=\frac{T(T-1)}{2}.
\label{delay-sum-eq}
\end{equation}\label{lemma-delay-sum}
\begin{proof}Note that all $\mathcal{D}_{\tau}$ for $\tau\geq t+D_{t}$ contain $t$. Thus, 
$$\sum_{t=1}^{T}|\mathcal{D}_{t}|=\sum_{t=1}^{T}\big[T+1-(t+D_{t})\big].$$
Since $|\mathcal{D}_{T}|=T$ and $D_{T}=0,$, the obtained expression is equivalent to desired equality \eqref{delay-sum-eq}.
\end{proof}
\end{lemma}

\begin{lemma}\label{fixed-share-prob-bound} Let ${N}_{T}$ be the sequence of experts $(n_{1},n_{2},\dots,n_{T})\in \mathcal{N}^{T}$, where $\mathcal{N}=\{1,2,\dots,N\}$. Let $p(\cdot)$ be the probabilistic model used in Fixed Share with prior $p_{0}\equiv \frac{1}{N}$ and switch probabilities $\alpha_{t}=\frac{1}{t}$ for all $t=2,\ldots,T$. Then 
$$-\ln p({N}_{T})\leq |d{N}_{T}+1|\cdot (\ln N+\ln T),$$
where $|d{N}_{T}|=|\{t:\,n_{t}\neq n_{t-1}\}|$ is the number of expert switches in $N_{T}$.
\begin{proof}Simple calculations
\begin{eqnarray}
-\ln p({N}_{T})=
\nonumber
\\
-\ln\frac{1}{N}-\sum_{t\in d{N}_{T}}\ln \frac{\alpha_{t}}{N}-\sum_{t\notin d{N}_{T}}\ln(1-\alpha_{t}+\frac{\alpha_{t}}{N})\leq
\nonumber
\\
-|d{N}_{T}+1|\cdot \ln\frac{1}{N}-\sum_{t\in d{N}_{T}}\ln \alpha_{t}-\sum_{t\notin d{N}_{T}}\ln(1-\alpha_{t})\leq\nonumber
\\
-|d{N}_{T}+1|\cdot \ln\frac{1}{N}-|d{N}_{T}|\cdot \ln \frac{1}{T}-\sum_{t=2}^{T}\ln\frac{t-1}{t}=
\nonumber
\\
|d{N}_{T}+1|\cdot (\ln N+\ln T)
\nonumber
\end{eqnarray}
prove the lemma.
\end{proof}
\end{lemma}

\end{document}